\documentclass[twoside,11pt]{article}

\usepackage{blindtext}

\usepackage[abbrvbib, preprint]{jmlr2e}
\usepackage{amsmath}
\usepackage{bm}
\usepackage{xcolor}
\usepackage{enumerate,enumitem}
\def\sph{\mathbb{S}} 

\usepackage{makecell}
\usepackage{multirow}
\usepackage{booktabs}    
\usepackage{tabularx}   

\newtheorem{thm}{Theorem}[section]       
\newtheorem{lem}{Lemma}[section]         
\newtheorem{prop}{Proposition}[section] 
\newtheorem{cor}{Corollary}[section]     
\newtheorem{rmk}{Remark}[section]        
\newtheorem{defi}{Definition}[section]        

\newcommand{\D}{{\cal D}}

\def\LS{\text{LS}}

\def\RR{\mathbb{R}}
\def\TT{\mathbb{T}}
\def\NN{\mathbb{N}}

\def\X{\mathcal{X}}
\def\relu1{\text{ReLU1}}
\def\relu{\text{ReLU}}

\def\nn{\mathcal{NN}}

\def\P{\mathcal{P}}

\def\C{\mathcal{C}}
\def\D{\mathcal{D}}

\def\PP{\mathbb{P}}

\def\t{\mathcal{T}}

\def\XX{\mathcal{X}}

\def\Span{\operatorname{Span }}
\def\Supp{\operatorname{Supp }}
\def\argmin{\operatorname{arg\,min }}
\def\esssup{\operatorname{ess\,sup }}

\newcounter{JFCounter}

\usepackage{lastpage}

\begin{document}
	
    \title{Sparse-Aware Neural Networks for Nonlinear Functionals: Mitigating the Exponential Dependence on Dimension}
	

\author{\name Jianfei Li \email lijianfei@math.lmu.de \\
	\addr Department of Mathematics\\
	Ludwig-Maximilians-Universit{\"a}t M{\"u}nchen (LMU Munich)\\
	Akademiestr. 7, 80799 M{\"u}nchen
  \AND
  \name Shuo Huang$^*$\email shuo.huang@iit.it \\
	\addr IIT@MIT\\
	Istituto Italiano di Tecnologia\\
	 Via Morego 30, 16163 Genova, Italy
	\AND
    \name Han Feng$^*$
		\email hanfeng@cityu.edu.hk \\
		\addr Department of Mathematics\\ City University of Hong Kong\\ Hong Kong, China
	\AND
		\name Ding-Xuan Zhou \email dingxuan.zhou@sydney.edu.au \\
		\addr School of Mathematics and Statistics\\ The University of Sydney, Sydney\\ New South Wales 2006, Australia
    \AND
    \name Gitta Kutyniok$^\dagger$ \email kutyniok@math.lmu.de \\
	\addr Department of Mathematics\\
	Ludwig-Maximilians-Universit{\"a}t M{\"u}nchen (LMU Munich)\\
	Akademiestr. 7, 80799 M{\"u}nchen
}

	\date{today}
	
	\maketitle
    \begingroup
    \renewcommand{\thefootnote}{\fnsymbol{footnote}}
    \footnotetext[1]{Corresponding author.}
    \begingroup
    \renewcommand{\thefootnote}{\fnsymbol{footnote}}
    \footnotetext[2]{Munich Center for Machine Learning (MCML), Germany. University of Troms{\o}, Norway. German Aerospace Center (DLR), Germany.}
    \endgroup

	\begin{abstract}
Deep neural networks have emerged as powerful tools for learning operators defined over infinite-dimensional function spaces. However, existing theories frequently encounter difficulties related to dimensionality and limited interpretability. This work investigates how sparsity can help address these challenges in functional learning, a central ingredient in operator learning. We propose a framework that employs convolutional architectures to extract sparse features from a finite number of samples, together with deep fully connected networks to effectively approximate nonlinear functionals. Using universal discretization methods, we show that sparse approximators enable stable recovery from discrete samples. In addition, both the deterministic and the random sampling schemes are sufficient for our analysis. These findings lead to improved approximation rates and reduced sample sizes in various function spaces, including those with fast frequency decay and mixed smoothness. They also provide new theoretical insights into how sparsity can alleviate the curse of dimensionality in functional learning.

\end{abstract}
	
	\section{Introduction}


        
        

Deep neural networks (DNNs) have demonstrated remarkable success in various domains, such as image processing \citep{lecun2015deep, krizhevsky2012imagenet} and natural language processing \citep{vaswani2017attention, devlin2018bert}, despite the extremely high dimensionality of the input data. 
More recently, they have emerged as effective tools for learning operators between infinite-dimensional function spaces, in contrast to classical function regression. A prominent example is image denoising, where the input and output are both images or functions considered \citep{greven2017general}. Another important application is to solve partial differential equations (PDEs) \citep{raissi2019physics, lu2021learning, li2020fourier}, where they have shown remarkable empirical success in learning nonlinear operators.

Despite their success in applications, theoretical guarantees in the context of operator learning remain limited. One of the earliest studies on operator learning is due to \citet{chen1995universal}, which investigates the conditions under which shallow neural networks can universally approximate any continuous nonlinear functional. The research has subsequently been extended to deep neural networks, as demonstrated in works such as \cite{li2020fourier,lu2021learning}. 
An essential step in learning operators is understanding functionals, namely mapping the input function into a finite-dimensional representation.
Common approaches include discretizing the input function via sampled function values~\citep{chen1995universal,lu2021learning} or linearly  truncated basis expansions ~\citep{yang2024spherical,song2023approximation(a),zhou2024approximation} that leverage linear approximation results. 
However, these results either lack explicit convergence rates or yield rates that are prohibitively slow due to the high dimensionality of the input \citep{lanthaler2022error,liu2025generalization,yang2025deeponet}, which contrasts with the strong empirical performance observed in practice.
This naturally leads to the question: What underlying principles and which forms of structured neural networks enable efficient operator learning?


An explanation for the success of deep neural networks (DNNs) lies in their powerful feature extraction capabilities; see, e.g., \citep{riesenhuber1999hierarchical,bengio2013representation}, which motivate the use of sparse representations. 
Classical sparse representation methods, including wavelets and shearlets \citep{mallat1999wavelet,labate2005sparse} or other data-driven dictionary learning approaches \citep{Ron2010Dictionaries,elad2010sparse}, leverage sparsity to represent objects using only a small number of nonzero coefficients.
More recently, it has been shown that the forward pass of convolutional neural networks (CNNs) is closely related to the sparse approximation, with recovery errors that can be quantitatively characterized \citep{papyan2017convolutional,li2024convergence}, highlighting the ability of deep learning techniques to represent sparse features.

These considerations motivate us to combine sparse approximation theory with CNN encoders in order to obtain improved learning rates that exhibit only mild dependence on the dimension for high-dimensional functional learning, i.e., mappings from infinite-dimensional function spaces to real numbers, under broad assumptions on the underlying dictionaries.
Our \textbf{contributions} are summarized as follows:

\begin{itemize}
    \item \emph{ Comprehensive Functional Learning Framework}: We introduce a unified theoretical framework to approximate nonlinear functionals over broad classes of function spaces (Theorem~\ref{thm:holder}). Our results provide flexibility in constructing encoders based on dictionaries that extend the commonly used orthogonal basis. The concrete results for various input domains are presented in Section \ref{sec:examples}.
    
    \item \emph{ Sparse Approximation via CNN Encoder}: Unlike many existing approaches that rely on linear encoders, we employ sparse-aware CNNs to perform nonlinear encoding. These CNN encoders derive sparse approximators from limited samples of input functions (Theorem~\ref{thm:nn-estimator}). Denoting $f_s$ as a CNN-based sparse approximator for $f$, it is demonstrated that $\|f-f_s\|_{L_p(\Omega,\nu)} 
		\leq c_1 e^{-c_2 J} + c_3 \sigma_s(F, \D_N)_{L_\infty(\Omega,\nu)}$, where $J$ controls the number of CNN parameters and $\sigma_s(F, \D_N)_{L_\infty(\Omega,\nu)}$ denotes the approximation error incurred when representing a function $f \in F$ using $s$ elements from the dictionary $\D_N$. Using a sufficient number of parameters, the CNN sparse approximator can potentially approximate $f$ accurately, provided that $f$ can be sparsely described by the dictionary $\D_N$.

    \item \emph{ Random Samples Suffice for Functional Learning}: In practical applications, we typically deal with finite samples from input functions. Yet, some current approaches require, for instance, grid points or cubature rules. Using cubature rules raises concerns about locating such samples in general manifolds or developing algorithms for general dictionaries. Lemma~\ref{lem:sparsity_estimation} and Lemma~\ref{thm:dai_samples} illustrate that random sampling of the locations of input functions meets the conditions of our general theoretical frameworks.
    
    \item \emph{Mitigating the Curse of Dimensionality}: Two function spaces are investigated with explicit error rates: those with rapid frequency decay and Sobolev functions with mixed smoothness. 
    Assuming that the input functions have smoothness of order $\alpha$ on a domain of dimension $d$, we obtain an approximation error of {\small $O \left( (\log K)^{-\beta(\alpha-\frac{3}{2})} (\log\log K)^{v(\alpha,\beta,d)} \right)$} for H\"{o}lder functionals with modulus of continuity $\omega_{\P}(r)\leq r^\beta$, when employing a neural network constrained to at most $K$ non-zero parameters.
  Notably, the dependence on the dimension $d$ appears only through the $\log\log K$ term, demonstrating that the curse of dimensionality can be alleviated in our setting.

\end{itemize}

The remainder of the paper is organized as follows. Section~\ref{sec:related} reviews related work, and Section~\ref{sec:pre} introduces the necessary preliminaries and notation. Section~\ref{sec:app_general} presents approximation results for general nonlinear continuous functionals, and Section~\ref{subsec:sampling} discusses sample size requirements for sparse approximation under different sampling schemes. Section~\ref{sec:examples} provides examples that yield explicit rates for several common input function spaces. Section~\ref{sec:conclusion} concludes the paper.

	\section{Related works}\label{sec:related}


\subsection{Input approaches for operator (functional) learning}
One of the major challenges in operator (functional) learning arises from the infinite or high dimensionality of the data. A straightforward approach is to use discrete function values sampled at grid points as inputs of neural networks; see, e.g., \cite{chen1995universal,lu2021learning}. In particular, \citet{chen1995universal} showed that shallow neural networks can achieve universality, and this result was later extended to deep architectures by \citet{lu2021learning}. An alternative to this mesh-dependent strategy is to employ truncated basis expansions, where the expansion coefficients serve as network inputs. Bases may be predefined, such as polynomials or Fourier bases \citep{song2023approximation(a),yang2024spherical,kovachki2021universal}, or data-dependent bases, as in \citet{liu2025generalization,yao2021deep}.

\subsection{Encoder-decoder structures in functional learning}
The strategy for learning from infinite-dimensional spaces is to extract its most crucial features.
Many existing works consider a domain-specific encoder $\phi:L_p(\Omega) \to \RR^N$, which first uses a linear operator $L_N$ to obtain the dominant part of the input function $f$ and discretizes it into a vector using a pre-selected dictionary $\{u_i \}_{i=1}^N$, that is,
\begin{align*}
    \operatorname{Encoder} &: \phi(f):= \left( \langle L_N f , u_1 \rangle, \langle L_N f, u_2 \rangle, \dots, \langle L_N f , u_N \rangle \right).
\end{align*}
The decoder is then constructed by a neural network for learning nonlinearity
\begin{align*}
    \operatorname{Decoder} &: \psi: \RR^N \rightarrow \RR.
\end{align*}
If $L_N f \in \Span \{u_i\}_{i=1}^N$ and $L_N f$ converges to $f$ as $N \rightarrow \infty$, then the encoder captures all the necessary information about $f$. The universal approximation theorem then ensures that the decoder is expressive enough to learn the associated nonlinear mapping. Consequently, we can anticipate that $\psi \circ \phi(f) \approx \P(f)$ for any nonlinear functional $\P(\cdot)$.

In \cite{yang2024spherical}, they consider $\P$ defined over Sobolev spaces $W^r_p(\mathbb{S}^d)$ with the spherical harmonic basis as the dictionary for linear approximation. Kernel embedding is considered to construct $L_N$ in \cite{shi2024nonlinear}. Due to the property of the Mercer kernel, the corresponding linear integral operator has an orthonormal basis of $L_2(\Omega)$. \cite{zhou2024approximation} considered nodal functions for linear approximation encoding while using tanh neural networks as decoder to approximate nonlinear functionals over reproducing Hilbert spaces. \cite{song2023approximation(a),song2023approximation(b)} used Legendre polynomials for the encoder part.

\subsection{Curse of dimensionality}
The approximation theory of deep learning has long faced the challenge of the curse of dimensionality. 
Roughly speaking, in function approximation, when the input dimension becomes large, the number of trainable parameters required to achieve a prescribed accuracy may grow exponentially. 
Many research efforts therefore focus on alleviating the impact of dimensionality \citep{montanelli2019new,liu2024approximation,LI2024SignReLU, zhou2020theory}. 
One common approach is to assume that the input data are obtained from some low-dimensional manifolds; see, e.g., \citep{liu2021besov,chui2018deep,shaham2018provable,schmidt2019deep,nakada2020adaptive,chen2019efficient,zhou2020theory}. 
Another strategy is to impose stronger assumptions on the input function space, e.g., \citep{montanelli2019new,liu2024approximation,LI2024SignReLU, zhou2020theory}. 
In functional learning, approximation rates are also often affected by the curse of dimensionality; see, e.g., \cite{shi2024nonlinear, song2023approximation(a), zhou2024approximation}. 
Inspired by the approximation theory of Barron spaces, several recent works extend dimension-free approximation ideas to specific classes of Barron-type functionals, obtaining rates that avoid the curse of dimensionality \citep{yang2022approximation,liuzhou2025generalization}.

\section{Preliminaries and Background}\label{sec:pre}
In this section, we summarize the notation, describe the structure of our neural networks, and present preliminary concepts and results of the sparse approximation.

\subsection{Notations}
Let $\RR$ and $\NN$ denote the sets of real numbers and positive integers, respectively, 
and define $\NN_0 := \NN \cup \{0\}$. 
For any $n \in \NN$, write $[n] := \{1, \dots, n\}$ and $[n]_0 := [n] \cup \{0\}$. 
The vector $\bm{1}_d$ represents a $d$-dimensional vector with all entries equal to~$1$. 
For a vector $z \in \RR^d$, we define its $p$-norm as $\|z\|_p = (\sum_i |z_i|^p)^{1/p}$  for $1 \le p <\infty$, $\|z\|_\infty = \max_i |z_i|$ if $p=\infty$, and $\|z\|_0 = \#\{ i : z_i \ne 0 \}$ as the number of nonzero entries.

Let $\nu$ be a probability measure of a compact domain $\Omega \subset \RR^d$. 
We denote by $\C(\Omega)$ the space of continuous functions on $\Omega$, 
equipped with the supremum norm $\|f\|_{\C(\Omega)} = \sup_{x \in \Omega} |f(x)|$. 
For $1 \le p < \infty$, the space $L_p(\Omega, \nu)$ consists of functions 
with finite $L_p$ norm,
\[
\|f\|_{L_p(\Omega, \nu)} := \left( \int_{\Omega} |f|^p \, d\nu \right)^{1/p} < \infty.
\]
When $p = \infty$, $L_\infty(\Omega, \nu)$ consists of functions 
with finite essential supremum norm,
\[
\|f\|_{L_\infty(\Omega, \nu)} := \esssup_{x \in \Omega} |f(x)|.
\]

For $r \in \NN_0$ and $\beta \in (0,1]$, the H\"{o}lder space $C^{r,\beta}(\Omega)$ is defined as the collection of functions that are $r$ times continuously differentiable and the $r$-th derivatives are $\beta$-H\"{o}lder continuous. 
The norm is given by
\begin{align*}
    \|f\|_{C^{r,\beta}(\Omega)} 
    := \sum_{\|\alpha\|_1 \le r} \|\partial^{\alpha} f\|_{\C(\Omega)}
    + \sum_{\|\alpha\|_1 = r} [\,\partial^{\alpha} f\,]_{C^{0,\beta}(\Omega)},
\end{align*}
where
\begin{align*}
    [\,f\,]_{C^{0,\beta}(\Omega)} 
    := \|f\|_{\C(\Omega)} 
    + \sup_{x \ne y} \frac{|f(x) - f(y)|}{\|x - y\|_2^{\beta}}.
\end{align*}
It is straightforward to verify that $C^{r,\beta}(\Omega) \subset C^{0,\beta}(\Omega)$.


\subsection{Functional neural networks}




Motivated by the encoder--decoder framework widely used in real practice, such as image processing tasks, we adopt a similar structure to address functional learning problems. 
In particular, our model consists of two components: a CNN that serves as the encoder to extract efficient finite-dimensional representations from high-dimensional inputs and a DNN that serves as the decoder to capture nonlinear relationships within the finite-dimensional domain.  In the subsequent section, we first introduce DNNs and CNNs, and then interpret the resulting encoder--decoder architecture as a tool for approximating nonlinear functionals.

\begin{definition}[Deep neural networks]\label{ch1:def:dnn}
    Given the input $x \in \RR^d$, a DNN $\psi$ with $L$ layers is defined iteratively by $\psi^{(0)}(x) = x$ and
    \begin{align*}
        \psi^{(\ell)}( x) &= \sigma \left ( W^{(\ell)} \psi^{(\ell-1)}( x) + b^{(\ell)}  \right ), \quad \ell=1, \ldots,L-1,\\
        \psi( x) &=  W^{(L)} \psi^{(L-1)}( x) +  b^{(L)},
    \end{align*}
    for some weight matrices $ W^{(\ell)} \in \RR^{d_{\ell} \times d_{\ell-1}}$ and bias vectors $ b^{(\ell)} \in \RR^{d_\ell}$ where $d_0 := d$.  Here $\sigma(u) := \max\{u, 0\}$ is the rectified linear unit (ReLU) activation for $u \in \RR$, acting entry-wise on vectors.
\end{definition}

We denote $\nn_{\operatorname{DNN}}(L, K)$ as the collection of functions generated by DNNs with no more than $L$ layers and at most $K$ nonzero neurons, that is, $K:= \#\{ W^{(\ell)}_{ij} \neq 0, b^{(\ell)}_i \neq 0, i \in [d_\ell], j \in [d_{\ell-1}], \ell \in [L] \}$.

CNNs can be viewed as a special type of DNNs, 
in which the usual fully connected weight matrices are replaced by structured convolutional operators. 
Specifically, for a one-dimensional convolution with kernel size $s$ and stride $t$, 
given a kernel $w \in \mathbb{R}^s$ and input $x \in \mathbb{R}^d$, the convolution operation is defined as
\[
(w *_t x)_i = \sum_{k=1}^{s} w_k\, x_{(i-1)t+k}, 
\qquad i = 1, \ldots, \mathcal{D}(d,t),
\]
where $\mathcal{D}(d,t) := \lceil d/t \rceil$ and we set $x_i := 0$ for $i > d$. 
Equivalently, there exists a matrix $T^{w,t} \in \mathbb{R}^{\mathcal{D}(d,t)\times d}$ determined by the kernel $w$ 
such that $T^{w,t}x = w *_t x$, see, for example, \cite{zhou2020universality,fang2020theory,li2024approximation}.
Multichannel CNNs are one of the most popular network architectures in practice and are defined as follows.
\begin{definition}[Multi-channel CNNs]\label{1dcnn}
For an input $x \in \mathbb{R}^d$, a $J$-layer multi-channel CNN block 
$\phi: \mathbb{R}^d \to \mathbb{R}^{d_J \times n_J}$ is equipped with a collection of filters $\{ w^{(j)}_{\ell,i} \in \mathbb{R}^{s_j} \}_{j=1}^J$. 
In the $j$-th layer, the kernel $w^{(j)}_{\ell,i}$, which has $s_j$ entries and uses stride $t_j$, connects the $i$-th input channel to the $\ell$-th output channel.
The CNN block $\phi$ is defined layer by layer $\{ \phi^{(j)}:\RR^d \rightarrow \RR^{d_j\times n_j} \}$ via the recursive relation
\begin{align*}
    \phi^{(j)}(x)_\ell 
        &= \sigma\!\left( 
        \sum_{i=1}^{n_{j-1}} w^{(j)}_{\ell,i} *_{t_j} \phi^{(j-1)}(x)_i 
        + b^{(j)}_\ell \bm{1}_{d_j} \right),
        && \ell \in [n_j], \quad j = 1, \ldots, J-1, \\
    \phi(x)_\ell 
        &= \sum_{i=1}^{n_{J-1}} w^{(J)}_{\ell,i} *_{t_J} \phi^{(J-1)}(x)_i 
        + b^{(J)}_\ell \bm{1}_{d_J},
        && \ell \in [n_J],
\end{align*}
where $\phi^{(0)}(x)=x$, $b^{(j)}_\ell \in \mathbb{R}$ are bias terms, and $d_j = \mathcal{D}(d_{j-1}, t_j)$ denotes the output dimension at layer $j$.

A CNN is built from a series of these multi-channel convolutional blocks, whose outputs are then flattened into vectors and passed on as inputs to the following layers.
\end{definition}


Similarly, let $\mathcal{N}_{\mathrm{CNN}}(J, M)$ denote the class of functions realized by CNNs with depth at most $J$ and at most $M$ nonzero parameters.
We are now ready to introduce functional neural networks for learning nonlinear functionals.

\begin{definition}[{Functional} neural networks]
    The collection of functional neural networks $\nn_\P(J,M,L,N)$ {for learning nonlinear functional $\P$} is defined as 
    \begin{align*}
        \nn_\P(J,M,L,K)
        &:= \left\{  \psi \circ \phi: \phi \in \nn_{\operatorname{CNN}}(c_1 J,c_2 M) , \psi \in \nn_{\operatorname{DNN}}(c_3 L, c_4 K)  \right\},
    \end{align*}
    where $\{c_i\}_{i=1}^4$ are constants that may vary from place to place. 
\end{definition}

\begin{remark}
    The input of $\nn_\P(J,M,L,K)$ is the discretized function values of the functions in the domain of $\P$. Similar structures can be found in other works, but with different encoders and different inputs. For example, in \cite{song2023approximation(a),song2023approximation(b),yang2024spherical}, the encoders are defined linearly without adaptivity, and the inputs are functions in infinite-dimensional spaces. In contrast, we emphasize the feature extraction capability of the encoder (CNNs) and discretized inputs, which is a more practical setting.
    For simplicity of notation, we omit the subscript $\P$ and denote the network as $\nn(J, M, L, K)$. 
\end{remark}





\subsection{Sparse approximation from finite samples}

Sparsity has emerged as one of the most fundamental structural properties of high dimensional data and has been empirically verified in a wide range of applications \citep{mallat1999wavelet,donoho2006compressed,tibshirani1996regression}. From the perspective of approximation theory, sparse properties are known to yield efficient approximation rates in function spaces and have been studied extensively in the context of nonlinear approximation, see, e.g. \cite{devore1993constructive,dai2023universal}. This section presents preliminary concepts and algorithm of sparse approximation for general function spaces.

Let $\Omega \subset \RR^d$ be a compact set with probability measure $\nu$ and $\D_N := \{ u_i \}_{i=1}^N \subset \C(\Omega)$ be a dictionary that consists of $N$ continuous functions. The collection of all the $s$-sparse combination of elements in dictionary $\D_N$ is denoted as
\begin{align*}
    \Sigma_s (\D_N) := \left\{ \sum_{i=1}^N w_i u_i \, \middle | \, w = (w_1, \ldots, w_N), \|w\|_0 \leq s  \right\}.
\end{align*}
For a Banach space $(X, \|\cdot\|_X)$, the best $s$-term approximation error of a functional data $f\in X$ and the worst-case error for a function set $ F\subset X$, with respect to $\D_N $, are defined respectively by
\begin{equation}\label{equ:best-app}
    \sigma_s(f,\D_N)_X := \inf_{g\in\Sigma_s(\D_N)} \|f-g\|_X,  \quad \sigma_s(F,\D_N)_X := \sup_{f\in F}\sigma_s(f,\D_N)_X.
\end{equation}

To approximate a target function $f$ by ones in $\Sigma_s(\D_N)$, one can equivalently search for a sparse coefficient vector $w \in \RR^N$ for the estimator associated with the dictionary $\D_N$. 
In practice, this task is typically formulated in a discretized setting. 
Given a set of samples $\xi := \{\xi_i\}_{i=1}^m \subset \Omega$, we denote the discretized function values and the dictionary matrix by
\[
\tilde{f}^m := (f\;|\;\xi)^\top:= (f(\xi_1), \dots, f(\xi_m))^\top,
\qquad 
\widetilde{D}_N^m := (\tilde{u}_1^m, \dots, \tilde{u}_N^m) \in \mathbb{R}^{m \times N}.
\]
The averaged empirical norms are given by
\begin{align*}
	\|\tilde{f}^m\|_{p,m} &:= \left( \frac{1}{m} \sum_{i=1}^m |f(\xi_i)|^p \right)^{1/p}, 
	\quad 1 \leq p < \infty,
\end{align*}
and
\(
\|\tilde{f}^m\|_{\infty,m} := \max_{i \in [m]} |f(\xi_i)|. 
\)
With these notations, the $s$-term estimator is obtained by solving
\begin{equation}\label{eq:s-coef}
    w_{p,m}^s 
  := \arg\min_{\substack{w \in \mathbb{R}^N \\ \|w\|_0 \le s}}
  \left\| f - \sum_{i=1}^N w_i u_i \;\middle|\;\xi \right\|_{p,m}
  := \arg\min_{\substack{w \in \mathbb{R}^N \\ \|w\|_0 \le s}} 
  \left\| \tilde{f}^m - \widetilde{D}_N^m w \right\|_{p,m}.
\end{equation}
Notice that, given $s$, $\D_N$, and $\xi$, the sparse coefficient vector $w_{p,m}^s$ changes with respect to different $f$ and can be viewed as an operator $w_{p,m}^s := w_{p,m}^s(f)$. For simplicity, we omit explicit dependencies and use $w_{p,m}^s$ throughout this paper.
The corresponding $s$-term estimator is then given by
\begin{equation}\label{eq:s-estimator}
    f_{p,m}^s := \sum_{i=1}^N \big(w_{p,m}^s\big)_i \, u_i.
\end{equation}

To ensure a good approximation result, we impose the following condition on the sampling set~$\xi$, which was also introduced in the literature; see, e.g., \cite{dai2023universal,dai2024random}.

\begin{ass}\label{ass:sparsecoding}
    We assume that the sampling set $\xi:= \{ \xi_j \}_{j=1}^m \subset \Omega$ provides universal discretization for $\Sigma_{2s}(\D_N)$, that is,  there exist constants $C_1, C_2 >0$, such that the following inequalities hold
    \begin{align}\label{def:universal_discretization}
			C_1 \|f\|_{L_p(\Omega, \nu)}^p \leq \frac{1}{m} \sum_{j=1}^{m} |f(\xi_j)|^p \leq C_2 \|f\|_{L_p(\Omega, \nu)}^p, \, \forall f \in \Sigma_{2s}(\D_N).
		\end{align} 
    In particular, the condition involving only the lower bound with constant~$C_1$ is called one-sided universal discretization. 
\end{ass}
This assumption imposes a regularity condition on the sampling set $\xi$, requiring that it provides a stable discretization of the $L_p(\Omega,\nu)$-norm. 
Such conditions arise in several areas, including norm discretization in approximation theory \citep{temlyakov2018universal}, sampling for numerical integration \citep{dick2013high}, and stable embedding conditions in compressed sensing \citep{candes2006stable}. In particular, it rules out poorly distributed sampling sets that could distort function norms when measured through discrete samples. We show that this assumption is mild and quantify the required sample size m for both deterministic and random sampling schemes in Section~\ref{subsec:sampling}.








The next lemma establishes an approximation bound for the sparse estimator of the function $f\in \C(\Omega)$. 
In fact, it holds for $\sigma_s(f, \D_N)_{L_p(\Omega, \nu_\xi)}$ with some probability measure $\nu_\xi$ for any $p \in [1, \infty]$, as shown in the Appendix~\ref{sec:s-apx}.

\begin{lemma}\label{lem:s_apx}
Let $m,s,N \in \mathbb{N}$ with $s \le N/2$ and $1 \le p \leq \infty$. 
Then for any $f \in \C(\Omega)$, we have
\[
\| f - f_{p,m}^s \;|\; \xi \|_{p,m} 
   \le 2^{1/p}\, \sigma_{s}(f, \D_N)_{L_\infty(\Omega,\nu)},
\]
In addition, if $\xi$ satisfies one-sided universal discretization (Assumption~\ref{ass:sparsecoding}), then
\[
\| f - f_{p,m}^s \|_{L_p(\Omega,\nu)} 
   \le C_3 \,\sigma_s(f, \D_N)_{L_\infty(\Omega,\nu)},
\]
where $C_3 > 0$ depends only on~$p$ as given in~\eqref{equ:cp}.
\end{lemma}

The above lemma is inspired by \cite{dai2023universal} and provides an approximation bound with respect to $\sigma_s(f,\D_N)$ when we learn sparse approximation from finite samples. 
The key difference compared to \cite{dai2023universal} lies in the algorithm: ours is purely based on a discretized representation of the function and dictionary. This modification makes our approach more practical for computation, and hence indicates that CNNs may serve as sparse approximators.

\section{Rates of approximating nonlinear functionals}\label{sec:app_general}

This section presents the approximation error of functional neural networks for functionals possessing certain smoothness properties on general continuous domains. We begin by analyzing the sparse approximation capability of the CNN part of functional neural networks.

\subsection{Sparse approximation via CNN encoders}
Algorithms to solve the sparse coding problem \eqref{eq:s-coef} are well studied in the literature. If $N \gg m$, the classical sparse coding problem aims to find the sparsest coefficient by solving the associated optimization problem
\[
\min_{w} \;\|w\|_0 \quad \text{s.t.} \quad \widetilde{f}^m = \widetilde{D}_N^m w.
\]
In the presence of noise, the basis pursuit denoising (BPDN) \citep{chen2001atomic} or the least absolute shrinkage and selection operator (LASSO) \citep{tibshirani1996regression} are typically adopted.
A large body of work has investigated the convergence behavior of these sparse coding algorithms, where a central role is played by the notion of {mutual coherence}, defined for a matrix $A \in \mathbb{R}^{m \times N}$ with columns $a_k$ as 
\begin{align*}
		\mu( A) = \max_{i\neq j} \frac{\left| a_i^\top a_j\right|}{\| a_i\|_2\|a_j\|_2}.
\end{align*}
It quantifies the maximum similarity between distinct columns: a smaller value indicates that the columns are nearly orthogonal (low redundancy). 
Moreover, mutual coherence provides fundamental guaranties for sparse recovery, since a smaller $\mu(A)$ allows one to recover signals of higher sparsity. In the following, we consider $m < N$ so that the discretized dictionary $\tilde{D}^m_N$ is redundant.
The book \citep{elad2010sparse} offers an in-depth treatment of these ideas and the associated sparse coding problems.
In recent years, algorithms for solving sparse coding problems \eqref{eq:s-coef} through deep learning have gained wide popularity under the paradigm of {learning to optimize}, see, e.g., \cite{gregor2010learning,de2024deep}. Inspired by this approach, we employ CNNs for sparse coding. Let
\begin{align}\label{notation:sbar}
    \bar{s} := \frac{1}{2}\left(1 + \frac{1}{\mu(\widetilde{D}_N^m)}\right).
\end{align}
Then next theorem specifies the associated approximation rate that can be attained by CNNs. The proof is given in Appendix~\ref{subsec:cnn-rcvy}.





\begin{theorem}\label{thm:nn-estimator}
	Let $m, s, k, J, N \in \mathbb{N}$, $1 \le p < \infty$, $B > 0$, and assume $s < \bar{s}$.
    Then there exists a CNN $\Phi \in \nn_{\operatorname{CNN}}(O(J\log_k(m+N)),J(m+N)^2) $ with kernel size $k$, such that for any $f\in F := \{ f \in \C(\Omega) : \|f\|_{\C(\Omega)} \le B \}$
    \begin{enumerate}
        \item $\| \Phi(\tilde{f}^m) \|_0 \leq s$;
        \item  $\left\| w_{p,m}^s - \Phi( \tilde{f}^m ) \right\|_{p} \leq   C_4 e^{-C_5 J} + C_6 \sigma_s(F, \D_N)_{L_\infty(\Omega,\nu)}$.
    \end{enumerate}
	In addition, if the one-sided universal discretization (Assumption~\ref{ass:sparsecoding}) holds, then 
	\begin{align}\label{equ:app-error}
		\|f-f_s\|_{L_p(\Omega,\nu)} 
		\leq   C_7 e^{-C_5 J} + C_8 \sigma_s(F, \D_N)_{L_\infty(\Omega,\nu)}.
	\end{align}
    where $f_s:= \sum_{i=1}^N \Phi( \tilde{f}^m )_i u_i $ and the constants $C_4,\dots,C_8$ depend on $p$, $m$, $B$, $s$, $N$, and $\tilde{D}^m_N$, as given in \eqref{eq:C4}, \eqref{eq:C5}, \eqref{eq:C6}, \eqref{eq:C7}, and \eqref{eq:C8} in Appendix~\ref{subsec:cnn-rcvy}.
\end{theorem}

The condition $s < \bar{s}$ indicates that the admissible sparsity level depends on the coherence of the dictionary. Intuitively, when dictionary elements are highly correlated, only very sparse representations can be reliably distinguished from sampled data, which imposes an upper bound on s. 
For instance, consider two orthogonal matrices $U, V \in \RR^{m \times m}$. Defining $A := [U, V]$, we obtain $1/\sqrt{m} \leq \mu(A) \leq 1$. Consequently, $s \leq (\sqrt{m} + 1)/2$, with equality achieved when $U$ is the identity matrix and $V$ is the Fourier matrix \cite{elad2010sparse}.
Such coherence-based sparsity bounds are classical in sparse recovery and ensure the uniqueness and stability of sparse representations.

The first error term in \eqref{equ:app-error} decays exponentially as the number of CNN layers increases, 
while the second term reflects the sparse approximation of $f$ with respect to the chosen dictionary $\mathcal{D}_N$. 
If $f$ admits an $s$-term representation using elements of $\mathcal{D}_N$, then this second term becomes small. 
In Theorem~\ref{thm:nn-estimator}, the sparsity condition on $s$ is imposed because, once the problem is discretized, this assumption guaranties the convergence of the sparse coding algorithm. Lemma \ref{lem:sparsity_estimation} characterizes this sparsity requirement in a specific setting.
Here, we do not explicitly state the expression for the sample size $m$, as it is implicitly contained in Assumption~\ref{ass:sparsecoding}. 
Section~\ref{subsec:sampling} provides the required sample sizes for deterministic and random sampling schemes for several specific dictionaries, and we shall show its advantages.


Theorem~\ref{thm:nn-estimator} is rooted in both approximation theory and sparse coding. The first universality result for CNNs was established in \cite{zhou2020universality}. In addition, approximation rates of CNNs have been systematically investigated across a range of function spaces; see, for instance, \cite{zhou2020universality} for smooth functions in Hilbert spaces, \cite{fang2020theory} for Sobolev spaces on the sphere, and \cite{zhou2020theory,mao2021theory,mao2023approximating,li2024approximation} for compositional function spaces that highlight their feature-learning capabilities. From the perspective of sparse coding, \cite{papyan2017convolutional} showed that convolutional neural networks can effectively solve sparse coding problems involving convolutional dictionaries, thereby revealing a fundamental link between CNNs and sparse coding algorithms. Subsequently, \cite{li2024convergence} generalized this result to broader classes of sparse coding problems with general activation functions, employing tools from approximation theory.
To the best of our knowledge, this is the first study to examine the sparse approximation capability of CNNs with respect to general dictionaries, which can be applied to a wide range of theoretical and practical settings.




\subsection{Error bound for learning H\"older continuous functionals}

\begin{figure}
    \centering
    \includegraphics[width=1.0\linewidth]{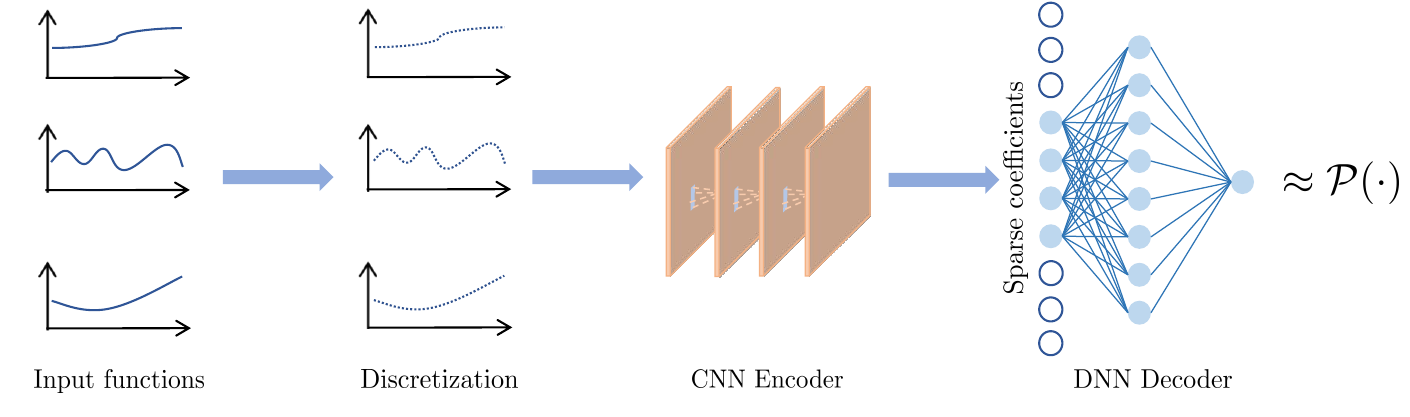}
    \caption{Functional learning pipeline proposed in Theorem~\ref{thm:holder} through functional neural networks. }
    \label{fig:holder}
\end{figure}

Consider a nonlinear functional $\mathcal{P}:  \C(\Omega) \to \mathbb{R}$. 
Its modulus of continuity is defined by
\begin{align*}
    \omega_{\mathcal{P}}(r;F)_p 
    := \omega_{\mathcal{P}}(r)_p
    := \sup\left\{|\mathcal{P}(f)-\mathcal{P}(g)| : \|f-g\|_{L_p(\Omega)} \le r, \; \text{for}\; f,g \in F \right\},
\end{align*}
where the dependence on $F$ is omitted when clear from the context.
The class of H\"older continuous functionals $C^{0,\beta}(\C(\Omega))_p$, $\beta \in (0,1]$ is given by
\begin{align*}
    C^{0,\beta}(\C(\Omega))_p
    := 
    \left\{
        \mathcal{P}: \C(\Omega)\rightarrow \RR : \omega_{\mathcal{P}}(r; \C(\Omega) )_p \le r^\beta
    \right\}.
\end{align*}

The following theorem establishes the approximation ability of functional neural networks for a 
nonlinear functional $\mathcal{P}$ defined on the function class 
\(
F = \{ f \in \C(\Omega) : \|f\|_{\C(\Omega)} \le B, B>0 \}.
\) The proof is given in Appendix~\ref{sec:holder}.

\begin{theorem}[H\"older continuous functional approximation]\label{thm:holder}
Let $m, s, K, M, N \in \mathbb{N}$, and let $B > 0$ and $\beta \in (0,1]$. 
Under Assumption~\ref{ass:sparsecoding}, for any nonlinear functional $\mathcal{P} \in C^{0,\beta}(\C(\Omega))_p$, 
there exists a functional neural network
\[
\Phi \in \nn\big(M \log(m+N),\, M(m+N)^2,\, \ln K,\, K\big)
\]
such that
\begin{align*}
    \sup_{f \in F}
    \left|
        \mathcal{P}(f) - \Phi(\widetilde{f}^m)
    \right|
    \le 
    \inf_{s <\bar{s}}
    \left\{
        C_{7}^{\beta} e^{-C_{5} \beta M}
        + C_{8}^{\beta} \big(\sigma_s(F,\mathcal{D}_N)_{L_\infty(\Omega,\nu)}\big)^{\beta}
        +C_9 (K/\ln K)^{-\beta/s}
    \right\},
\end{align*}
where $\bar s$ is given in \eqref{notation:sbar} and the constants $C_5,C_7,C_8, C_9$ depend on $p$, $m$, $B$, $s$, $N$, and $\tilde{D}^m_N$ as given in \eqref{eq:C5}, \eqref{eq:C7}, \eqref{eq:C8}, and \eqref{eq:C9} 
in Appendix~\ref{subsec:cnn-rcvy}.
\end{theorem}

    
    

Figure~\ref{fig:holder} provides a graphical depiction of how neural networks learn functionals as described in Theorem~\ref{thm:holder}.
The first and last terms of the upper bound in Theorem~\ref{thm:holder} capture the learning capacities of the CNN encoder and the DNN decoder in the functional network. As their respective numbers of parameters increase, the encoder achieves an exponential decay in error, whereas the decoder exhibits a polynomial decay. The second term depends on the choice of the dictionary, and we provide more explicit rates for several commonly used dictionaries and function spaces in Section~\ref{sec:examples} (Corollary~\ref{cor:main_decay} and \ref{cor:main_mixed_sobolev}). Moreover, in many practical applications, such as image processing, this term is typically very small.

\begin{remark}[Curse of dimensionality]\label{rem:cod}
    The error in the last term decays at the rate $(K/\ln K)^{-\beta/s}$, which depends on the sparsity level $s$ rather than on $N$. 
This is advantageous because $N$ typically grows exponentially with dimension $d$ when $\Omega \subset \mathbb{R}^d$; 
for example, the space of polynomials of degree $n$ over $[-1,1]^d$ contains $O(n^d)$ basis functions, 
which leads to extremely slow decay rates when the dependence is on $N$.
\end{remark}

\section{Sample size requirements for various sampling schemes}\label{subsec:sampling}
As mentioned in Section~\ref{sec:app_general}, the requirement of sample locations and dictionary is contained
in the Assumption~\ref{ass:sparsecoding}, which ensures that the $L_p$ norm is equivalent to 
the discrete $\ell_p$ norm on the sample set $\{\xi_j\}_{j=1}^m$ for all 
$f \in \Sigma_{2s}(\mathcal{D}_N)$. 
In this section, we demonstrate how many samples are needed for 
Assumption~\ref{ass:sparsecoding} to hold for several commonly used dictionaries 
(such as orthogonal bases), under both 
deterministic and random sampling schemes. 
As we show later, the sparsity pattern enables a substantial decrease in the necessary sample size. We begin by specifying the required properties of the dictionary.

\begin{ass}\label{ass:DN}
Assume that $\mathcal{D}_N$ satisfies:
\begin{enumerate}[label=(\theass\alph*),ref=\theass\alph*]
    \item $\sup_{x\in\Omega} |u_i(x)| \le 1$, $i \in [N]$. \label{ass:bndD}
    \item There exists $R>0$ such that 
    $\sum_{i=1}^N |w_i|^2 \le R \bigl\|\sum_{i=1}^N w_i u_i\bigr\|_2^2$ 
    for any $w \in \mathbb{C}^N$. \label{ass:riesz}
\end{enumerate}
\end{ass}

As noted in~\cite{dai2024random}, Assumption \ref{ass:riesz}, also called Riesz basis, is relatively mild, as it can always be satisfied by choosing $R$ as the reciprocal of the smallest eigenvalue of the matrix $D$, whose entries are defined by $D_{ij} = \langle u_i, u_j \rangle_{L_2(\Omega)}$. In particular, this condition rules out dictionaries containing highly correlated or nearly linearly dependent elements, which could allow large coefficient vectors to cancel each other and produce functions with very small norm and make the representation unstable. Many commonly used dictionaries satisfy Assumption~\ref{ass:DN}.
Examples include orthonormal systems such as Fourier bases, spherical harmonics, and orthogonal polynomial systems (e.g., Legendre or Chebyshev polynomials), for which $R=1$. It is also satisfied by compactly supported wavelet systems, B-splines, and finite element bases (such as piecewise linear or higher-order hat functions); see, e.g., \cite{cohen1992biorthogonal}, \cite{brenner2008mathematical}.

\subsection{Deterministic sampling}
Deterministic sampling provides an explicit way to guarantee the universal 
discretization property required in Assumption~\ref{ass:sparsecoding}. 
When the dictionary satisfies  Assumption~\ref{ass:DN}, 
the recent result of \cite{dai2023universal} shows that one can choose a relatively 
small number of deterministic sample points to achieve this property. The following lemma 
summarizes this result.

\begin{lemma}[\cite{dai2023universal}]\label{thm:dai_samples}
    Let {$1\leq p \leq 2$}. If the dictionary $\D_N$ satisfies Assumptions \ref{ass:DN}, then for any $1\leq s \leq N/2$, there exist $\xi \in \Omega^m$ with
    \begin{align*}
        m \leq Cs \log N \cdot \log^2(4s) (\log (4s)+\log\log N),
    \end{align*}
    such that $\xi$ provides universal discretization (Assumption~\ref{ass:sparsecoding}) with $C_1=1/2$ and $C_2 = 3/2$.
\end{lemma}

The above lemma shows that the required sample size is of order $O(s\log^3 s\cdot\log N\cdot\log\log N)$, which is 
substantially smaller than those in related works whenever $s \ll N$.
For example, orthonormal bases with uniformly bounded $L_\infty$ norms and 
nodal bases used in interpolation satisfy Assumption~\ref{ass:DN}, and therefore 
our results apply directly to these settings. 
In contrast, in previous work using spherical harmonics, the number of samples must scale 
linearly with $N$ \citep{yang2024spherical,feng2023generalization}. 
Similarly, \cite{zhou2024approximation} employs a fixed uniform grid for interpolation 
in a reproducing kernel Hilbert space during the discretization step, where again the 
sample size is required to be equal to $N$.


\subsection{Random sampling}
While deterministic constructions provide existence guaranties and valuable theoretical insight, they are often challenging to realize in practice. 
Therefore, we next show that random sampling achieves the same guaranties with high probability 
and with a comparable sample size. The following lemma also appears in \cite{dai2024random}, and therefore we omit the proof here.

\begin{lemma}[\cite{dai2024random}]\label{thm:dairandom}
    Assume that $\mathcal{D}_N$ satisfies Assumption~\ref{ass:DN}.  
Let $\xi$ be a set of $m$ independent and identically distributed random points on $\Omega$ 
with probability distribution $\nu$. Then for any $1 \le p \le 2$ and any integer 
$1 \le s \le N/2$, the sample set $\xi$ satisfies Assumption \ref{ass:sparsecoding} with constants $C_1 = \frac{1}{2}$ and $C_2 = \frac{3}{2}$, 
with probability at least
$ 1-2 \exp\{-cm/ (R s \log^2(4Rs))\}$
provided that
\[
    m \ge C R s \log N \cdot (\log(4 R s))^2 \cdot \bigl( \log(4 R s) + \log\log N \bigr),
\]
where $c$ and $C$ are positive constants depending only on $p$.
\end{lemma}

\subsection{Random sampling under orthogonal systems}
The results provided above give reasonable estimates for the required sample size and for Assumption~\ref{ass:sparsecoding}. However, they do not suffice to produce an explicit convergence rate when applied to Theorem~\ref{thm:holder}, which requires the estimation of mutual coherence. 
In the following, we focus on orthogonal dictionaries and derive explicit formulas to estimate the sparsity levels that still ensure Assumption~\ref{ass:sparsecoding}.



\begin{ass}\label{ass:ortho}
Assume that $\mathcal{D}_N$ is an orthogonal system, i.e.,
  $$\langle u_i, u_j \rangle_{L_2(\Omega,\nu)}
    := \int_{\Omega} u_i(x)\, u_j(x)\, d\nu(x)
    = \gamma \,\delta_{ij}, \; i,j \in [N],$$
where $\gamma>0$ is a constant independent of $N$, and $\delta_{ij}=1$ if $i=j$ 
and $0$ otherwise.
\end{ass}

It is clear that Assumption \ref{ass:ortho} yields the Riesz-type condition (Assumption~\ref{ass:riesz}).  
A simple example of a dictionary $\mathcal{D}_N$ that satisfies both Assumption \ref{ass:bndD} and Assumption \ref{ass:ortho}
is a truncation of the real trigonometric system
\[
    \left\{ 1,\; \sqrt{2}\cos(2\pi k t),\; \sqrt{2}\sin(2\pi k t) \right\}_{k=1}^{\infty},
\]
which is orthonormal on $\Omega=[0,1]$. For the $d$-dimensional real trigonometric system, we can normalize the functions to satisfy Assumption~\ref{ass:bndD} and Assumption~\ref{ass:ortho}. In this case, $\gamma = 2^{-d}$.

Under this orthogonality condition, the next lemma establishes not only the necessary sample size but also an explicit upper bound on the admissible sparsity level $s$. A detailed proof is provided in Section~\ref{sec:pf_6}.

\begin{lemma}\label{lem:sparsity_estimation}
Assume that Assumptions \ref{ass:bndD} and \ref{ass:ortho} hold.
Let $\xi$ be a set of 
independent and identically distributed random points on $\Omega$ with distribution $\nu$.  
Then, with probability at least $1-\varepsilon$, the following statements hold whenever $m > \frac{64}{3\gamma} \log\left(2N^2/\varepsilon\right)$:
\begin{enumerate}
    \item Assumption \ref{ass:sparsecoding} holds with $C_1=\frac{1}{4}$ and $C_2=\frac{9}{4}$.
    \item The mutual coherence is bounded by \[\mu(\tilde{D}_N^{\,m}) 
\le
8\sqrt{\frac{ \log(2N^2/\varepsilon)}{3\gamma m}}.\]
\item The sparsity can be chosen as \[ s
=  \left\lfloor\frac{1}{2} \left( 1+ \frac{1}{16} \sqrt{\frac{3\gamma m}{\log(2N^2/\varepsilon)}} \right) \right\rfloor 
\le
\bar s,\]
\end{enumerate}
where $\bar s$ is defined in \eqref{notation:sbar}.
\end{lemma}

    


\section{Explicit upper bounds for specific input function spaces}\label{sec:examples}
In this section, we present explicit learning rates for approximating nonlinear functionals, as formulated in Theorems~\ref{thm:holder}, for several distinct choices of input function spaces. This is made possible by the estimate derived in Lemma~\ref{lem:sparsity_estimation}. Within our framework, we will show that deep neural networks circumvent the curse of dimensionality, thereby providing an explanation for their strong empirical performance in high-dimensional settings.

In what follows, alleviating the curse of dimensionality means that when the total number of parameters in functional neural networks is $K$, the required sample size $m$ and the error tolerance $\varepsilon$, expressed as functions of $K$, do not deteriorate at an exponential rate with respect to the dimension $d$ of the input function domain.

\subsection{Function spaces with fast coefficient decays under dictionaries}

We first consider functions that have fast decay coefficients with $\alpha>0$ in the dictionary $\D_N$:
\begin{align*}
    A_1^\alpha(\D_N) := \left \{ \sum_{i=1}^N c_i u_i: \sum_{i=1}^N |c_i| i^\alpha \leq 1 \right \}.
\end{align*}
The function class $A_1^\alpha(\mathcal D_N)$ naturally arises from several classical
function spaces after an appropriate ordering of the dictionary elements.
In the following, we illustrate this with Sobolev spaces and Besov spaces on the sphere.

\medskip
\noindent
\textbf{Example (Sobolev spaces on the sphere).}
Let $f \in W_p^r(\mathbb S^{d-1})$ be a Sobolev function on the unit sphere.
With respect to the spherical harmonic basis
$\{Y_{k,\ell} : k \in \mathbb Z_+,\, \ell=1,\dots,C(k,d)\}$,
its Sobolev norm admits the characterization
\begin{align*}
\|f\|_{W_p^r(\mathbb S^{d-1})}
\asymp
\sum_{k=0}^\infty
k^r
\left(
\sum_{\ell=1}^{C(k,d)}
|\widehat f_{k,\ell}|^p
\right)^{1/p},
\end{align*}
where
$\widehat f_{k,\ell}
= \langle f, Y_{k,\ell} \rangle_{L_2(\mathbb S^{d-1})}$
and $C(k,d) \asymp k^{d-2}$.

Fix an ordering of the spherical harmonics
$\{Y_{k,\ell}\}_{k\ge0,\;1\le\ell\le C(k,d)}$
by non-decreasing degree $k$, and label them as a single-indexed dictionary
$\mathcal D := \{u_i\}_{i\ge1}$.
With this ordering, the total number of basis functions up to degree $k$ satisfies
\[
i
=
\sum_{j=0}^k C(j,d)
\asymp k^{d-1}.
\]
Hence, for coefficients corresponding to the spherical harmonics of degree $k$, the index $i$
satisfies
\[
k^{d-2} \lesssim i \lesssim k^{d-1},
\quad\text{or equivalently}\quad
k^{(d-2)/(d-1)} \lesssim i^{\,1/(d-1)} \lesssim k.
\]
Under this identification, the Sobolev norm ($p=1$) can be rewritten as
\begin{align*}
\|f\|_{W_1^r(\mathbb S^{d-1})}
\lesssim
\sum_{i=1}^\infty
i^{\frac{r}{d-1}}
|c_i|,
\qquad
c_i := \langle f, u_i \rangle_{L_2(\mathbb S^{d-1})},
\end{align*}
which shows that the unit ball of $W_1^r(\mathbb S^{d-1})$
is continuously embedded in
$A_1^\alpha(\mathcal D)$ with $\alpha = \frac{r}{d-1}$.

\noindent
\textbf{Example (Besov spaces on the sphere).}
Let $\Omega = \mathbb{S}^{d-1}$ and $\{\psi_{\eta}\}$ be a needlet system that can be a tight frame for $L_2(\mathbb{S}^{d-1})$, i.e., $f = \sum_{\eta \in \chi } \langle f, \psi_\eta \rangle \psi_{\eta} $ with $\chi = \cup_{j=0}^\infty \chi_j$. Here $\chi_j$ is a set of points almost uniformly distributed on $\mathbb{S}^{d-1}$ and $\#\chi_j \approx 2^{j(d-1)}$. 
Suppose that $f$ belongs to a Besov space $B^{\alpha,q}_p(\mathbb{S}^{d-1})$ with $\alpha > 0$. 
Then its expansion coefficients satisfy
\[
    \|f\|_{B^{\alpha,q}_p} 
    \approx
    \left( 
        \sum_{j=0}^{\infty} \left[
        2^{j(\alpha+(d-1)/2- (d-1)/p)} 
        \left( 
            \sum_{\eta\in \chi_j} |\langle f, \psi_\eta \rangle|^p 
        \right)^{1/p} \right]^q
    \right)^{1/q}, \forall f\in B^{\alpha,q}_p(\mathbb{S}^{d-1}).
\]
The number of elements up to scale $j$ satisfies $ \sum_{\ell=0}^{j} 2^{\ell (d-1)} = O(2^{j d})$. That is, for the nonincreasing rearrangement $\{\bar{f_i}\}$ of $\{\langle f, \psi_\eta \rangle\}$ with respect to $j$, the index $i$ for the elements of scale $j$ satisfies 
\[ 2^{j(d-1)} \lesssim i \lesssim 2^{jd},
\quad\text{or equivalently}\quad 2^{j(d-1)/d} \lesssim i^{\frac{1}{d}} \lesssim 2^{j}.\]
In particular, for $p=q=1$, one has
\[
    \|f\|_{B^{\alpha,1}_1} 
    \approx 
    \sum_{j=0}^{\infty} 2^{j(\alpha-(d-1)/2)} \sum_{\eta\in\chi_j} |\langle f, \psi_\eta \rangle|,
\]
which implies that
$$\sum_i i^{\alpha/d-(d-1)/2d} |\bar f_i| \lesssim \|f\|_{B_{1,1}^s}.$$
Therefore, the unit ball of $B^{1,\alpha}_1$ can be viewed as a special case of $A_1^{\alpha/d-(d-1)/2d} (\{\psi_\eta\})$ when $\alpha/d-(d-1)/2d>0$. For a more in-depth discussion, see \cite{narcowich2006decomposition}.

In what follows, we examine broad classes of function spaces whose coefficients, relative to specific dictionaries, decay at a similarly fast rate. We demonstrate that the subsequent result avoids the curse of dimensionality when learning nonlinear functionals.
\begin{corollary}\label{cor:main_decay}
       Consider $F \subset A_1^\alpha(\D_N))_{\infty}$ with $\alpha> \frac{3}{2}$ and $\D_N$ which satisfies Assumptions \ref{ass:bndD} and \ref{ass:ortho}. Let $\xi$ be a set of $m$ independent and identically distributed random points on $\Omega$ with a probability distribution $\nu$. If $m \asymp \left( \log K \right)^2/\log \log K$ for an integer $K$, then with probability at least $$1-\log \log K\cdot \left( \log K \right)^{-2},$$
       for any nonlinear functional $\P\in C^{0,\beta}(\C(\Omega))_{2}$, there exists a functional neural network 
       $$\Phi \in \nn\left( \left(\log\log K \right)^2,\left(\log K \right)^2, \log K, K\right)$$ such that
    \begin{align*}
        \sup_{f\in F}|\P(f) - \Phi(\tilde{f}^m)| =
        O \left( (\log K)^{-\beta(\alpha-\frac{3}{2})} (\log\log K)^{\beta(\alpha-1)}  \right).
    \end{align*}
\end{corollary}

This result achieves an approximation rate that decays exponentially in $\log K$ and remains free of any dependence on the dimension of the input domain. In contrast, the earlier rate of order $(\log K)^{-\alpha/d}$, see, e.g. \cite{song2023approximation(a),yang2024spherical,shi2024nonlinear}, is much slower and is affected by the curse of dimensionality.

\subsection{Sobolev functions with mixed smoothness}

The second example of the input function class consists of periodic functions with mixed derivatives, which play a central role in 
high-dimensional approximation theory, see, e.g., \cite{temlyakov2015constructive,dai2023universal}. Let $\TT$ be the torus represented by $[0,2\pi]$.  In the following, we use 
$\|\cdot\|_{L_p} := \|\cdot\|_{L_p(\TT^d)}$ for simplicity.

   

To analyze approximation properties, it is convenient to decompose a function 
into dyadic ``mixed’’ frequency blocks.  
For $f \in L_1(\TT^d)$, define
\[
f_j
= \sum_{\|\mathbf n\|_1=j}
\;\sum_{\lfloor 2^{n_i-1}\rfloor \le |k_i| < 2^{n_i}}
\widehat f(k)\, e^{i k\cdot x},
\qquad \widehat f(k)= (2\pi)^{-d} \int_{\TT^d} f(x)\, e^{-i k\cdot x} dx,
\]
where $j\in\mathbb N_0$ and $k=(k_1,\dots,k_d)$ denotes a frequency vector and 
$\|\mathbf n\|_1=n_1+\cdots+n_d$ controls the mixed dyadic level, rather than the isotropic scale $\max_i n_i$.

We define the mixed-smoothness classes
\[
W_A^{a,b}
:= \left\{
   f : \|f_j\|_A \le 2^{-aj}\, \bar j^{(d-1)b}
   ,
\bar j := \max\{j,1\}, j\in \NN \right\},
\]
where
\[
\|f\|_A := \sum_{k\in \mathbb Z^d} |\widehat f(k)|
\]
is the Wiener algebra norm (sum of absolute Fourier coefficients). 

This example is closely connected to $A_1^\alpha(\D_N)$ because its Fourier coefficients decay rapidly. A similar result is given in the following corollary.

\begin{corollary}\label{cor:main_mixed_sobolev}
Consider
$F \subset W_A^{a,b}$ with $\|f\|_{\C(\TT^d)}\leq 1$, $a > \tfrac{3}{2}$, and $b \in \mathbb{R}$. Let $\xi$ be a set of $m$ independent and identically distributed random points on $\Omega$ with probability distribution $\nu$. If $m \asymp \left( \log K \right)^2/\log \log K$ for an integer $K$, then with probability at least $$1-\log \log K\cdot \left( \log K \right)^{-2},$$
       for any nonlinear functional $\P\in C^{0,\beta}(\C(\TT^d))_{2}$, there exists a functional neural network 
       $$\Phi \in \nn\left( \left(\log\log K \right)^2,\left(\log K \right)^2, \log K, K\right)$$ such that
    \begin{align*}
        \sup_{f\in F}|\P(f) - \Phi(\tilde{f}^m)| =
        O \left( (\log K)^{-\beta(a-\frac{3}{2})} (\log\log K)^{\beta(a+(d-1)(a+b)-\frac{1}{2})}  \right).
    \end{align*}
\end{corollary}

    \subsection{Discussion}


\begin{table}[t]
  \centering
  \caption{Comparison of approximation rates of H\"older continuous functionals across various input function spaces and encoding schemes. We employ $\tilde{O}$ notation in place of standard $O$ notation to emphasize the leading term, suppressing the $\log\log$ factor.}
  \label{tab:comparison_rates}
  \resizebox{\textwidth}{!}{
  \begin{tabularx}{\textwidth}{>{\raggedright\arraybackslash}X >{\centering\arraybackslash}X >{\centering\arraybackslash}l}
    \toprule
    Input functions & Encoder & Rate \\
    \midrule
    $W_p^\alpha(\sph^{d-1})$ \citep{yang2024spherical} & Spherical harmonics & \( \tilde{O}\left( (\log K)^{-\beta\alpha/d-1} \right) \) \\
    \midrule
    $C^\alpha([-1,1]^d)$ \citep{song2023approximation(a)} & \multirow{2}{*}{\makecell[c]{Legendre polynomials}} & \(\tilde{O}\left((\log K)^{-\beta \alpha/d} \right) \) \\
    Analytic functions \citep{song2023approximation(b)} & & $\tilde{O} \left( \exp(-c (\log M)^{\frac{1}{d+1}} ) \right) $ \\
    \midrule
    $B^{\alpha}_{2,\infty}(\RR^d), H_0^\alpha([0,1]^d)$ \citep{shi2024nonlinear} & \multirow{2}{*}{\makecell[c]{Eigensystem induced by \\ Mercer kernel}} & \(\tilde{O}\left((\log K)^{-\beta \alpha/d} \right) \) \\
    Gaussian RKHS \citep{shi2024nonlinear} & & $\tilde{O} \left( \exp(-c (\log M)^{\frac{1}{d+1}} ) \right) $ \\
    \midrule
    $A_1^\alpha(\D_N)$ (\textbf{ours}) & CNN & \( \tilde{O}\left( (\log K)^{-\beta(\alpha-3/2)} \right) \) \\
    $W_A^{a,b}$ (\textbf{ours}) & CNN & \( \tilde{O}\left( (\log K)^{-\beta(a-3/2)} \right) \) \\
    \bottomrule
  \end{tabularx}
  }
\end{table}

In the literature, the convergence for approximating a functional is typically rather slow.
In \cite{song2023approximation(b),song2023approximation(a),yang2024spherical}, 
the authors consider orthogonal basis truncation or adaptive bases obtained through kernel embeddings \citep{shi2024nonlinear}, 
using the resulting coefficients as input to subsequent neural networks. 
For nonlinear functionals defined on Sobolev spaces $W_p^r(\Omega)$, $\Omega \subset \mathbb R^d$, 
they obtain rates of order $(\log K)^{-r/d}$ in terms of the number of free parameters $K$ in neural networks. 
This coincides with the minimax lower bound for continuous functionals on $L_p$ without further assumptions \citep[Theorem 2.2]{mhaskar1997neural}. Beyond Sobolev spaces, faster rates are achieved in smaller input spaces. 
For analytic functions, \cite{song2023approximation(a)} establish rates of 
$\exp\!\bigl(-c (\log K)^{1/(d+1)}\bigr)$; 
for some reproducing kernel Hilbert spaces (RKHS), \cite{zhou2024approximation} obtain 
$(\log K)^{-s(2r-d)/d}$ with $2r-d>2$; 
and for Barron spectral spaces of functionals, \cite{yang2022approximation} derive a rate of $K^{-1/2}$.  Similar bounds have also been established in the operator learning literature. For instance, \cite{lanthaler2022error} derive rates for analytic function inputs, 
while \cite{kovachki2021universal,liu2025generalization} consider inputs from Sobolev spaces. 

In contrast, our study investigates sparse approximation with respect to general dictionaries. In Table~\ref{tab:comparison_rates}, we report both the approximation results previously established in the literature and those obtained in this work. Since the output of the CNN encoder is sparse, the decoder can exploit this sparsity and, therefore, requires significantly fewer neurons. We demonstrate these benefits for input functions with rapidly decaying coefficients and for Sobolev spaces with mixed smoothness, where the curse of dimensionality can be alleviated. These findings suggest that deep neural networks have the capacity to handle high-dimensional functional data. 


\section{Conclusion and future work}\label{sec:conclusion}


In this paper, we develop a general functional learning framework based on sparsity-aware neural networks. Our framework has three main advantages: (i) it accommodates arbitrary dictionaries; (ii) it shows that random sampling is sufficient to recover sparsity from input functions; and (iii) it establishes that CNN encoders can serve as effective sparse approximators.

Within this general framework, we consider two classes of input functions: Sobolev spaces with mixed derivatives and function spaces with fast frequency decays. Exploiting the fact that functions in these classes can be well approximated using only a few elements from suitable dictionaries, we prove that our approximation bounds are independent of the ambient dimension $d$ (the dimension of the domain of the input functions). These function classes are reasonably large, and our results demonstrate that CNN-based architectures can effectively mitigate the curse of dimensionality in functional learning problems.

Since the approximation of nonlinear functionals is a fundamental ingredient in operator learning, as noted in the literature (see, e.g., \cite{mhaskar1993approximation,liu2025generalization}), it is natural to extend the present framework to the operator learning setting. Another important direction is to investigate the generalization properties of our encoder–decoder architecture, which involves not only approximation guaranties but also a detailed analysis of sample estimation error. 

    \section*{Acknowledgement}
    \addcontentsline{toc}{section}{Acknowledgement}
    
    J. Li gratefully acknowledges support from the project CONFORM, funded by the 
    German Federal Ministry of Education and Research (BMBF), as well as from the 
    German Research Foundation under the Grant DFG-SPP-2298. Furthermore, J. Li 
    acknowledges additional support from the project ``Next Generation AI Computing 
    (gAIn)'', funded by the Bavarian Ministry of Science and the Arts and the Saxon 
    Ministry for Science, Culture, and Tourism.
    The work of H. Feng was partially supported by CityU 11315522; CityU 11300525.  The work of D. X. Zhou described in this paper was partially supported by Discovery Project (DP240101919) of the Australian Research Council.
    The work of G. Kutyniok was supported in part by the Konrad Zuse School of 
    Excellence in Reliable AI (DAAD), the Munich Center for Machine Learning (MCML), 
    as well as the German Research Foundation under Grants DFG-SPP-2298, KU 1446/31-1 
    and KU 1446/32-1. Furthermore, G. Kutyniok acknowledges additional support from 
    the project ``Next Generation AI Computing (gAIn)'', funded by the Bavarian 
    Ministry of Science and the Arts and the Saxon Ministry for Science, Culture, 
    and Tourism, as well as by the Hightech Agenda Bavaria.
    

	
	

	
    \bibliography{sample}

    \newpage
\appendix


\section{Proof of Lemma~\ref{lem:s_apx}}\label{sec:s-apx}

In this section, we first rewrite the problem \eqref{eq:s-coef} into a two-step formulation, which helps to clarify the differences from the work of \cite{dai2023universal}. 
For completeness, some of the relevant notation is recalled. 
Furthermore, we extend Lemma~\ref{lem:s_apx} to general probability measures, which yields an upper bound with the best $s$-term approximation in the $L_p$ norm for all $1 \le p \le \infty$, rather than only for $p=\infty$. 

We denote the collection of $s$-dimensional linear spaces that are spanned by arbitrary $s$ elements from $\D_N$ as
\begin{align*}
	\X_s(\D_N) := \left \{ V: \text{dim}(V) = s, V \subset \Span \D_N \right \}.
\end{align*}
The $p$-norm of vectors is defined as
\begin{align*}
    \|x\|_p &= \left( \sum_{i} |x_i|^p \right)^{1/p}, \quad 1 \leq p < \infty,
\end{align*}
and $ \|x\|_\infty = \max_i |x_i|$ for $p=\infty$.
In regression, learning $f$ from its samples and a linear space $V$ can be formulated as
\begin{align}\label{eq:ls}
    \LS_p (\xi, V)(f) := \argmin_{v \in V} \left \|f-v | \xi  \right \|_{p,m}^p.
\end{align}
Searching over all $V \in \X_s(\D_N)$, we choose a linear space on which we achieve the smallest error over $\xi$ and it is easy to see that
\begin{align}\label{eq:optimal_subspace}
	f^s_{p,m} = \argmin_{V\in \X_s(\D_N) }\left \|f - \LS_p(\xi,V)(f) |\xi  \right \|_{p,m}.
\end{align}
We denote the best approximation of $f\in L_p(\Omega, \nu)$, $1\leq p \leq \infty$ by elements in an $N$-dimension subspace $X_N$ of $\C(\Omega)$ as follows
\begin{align*}
	d(f, X_N)_{L_p(\Omega, \nu)} := \inf_{u\in X_N} \|f- u \|_{L_p(\Omega, \nu)}.
\end{align*}
In the proof, we also use the probability measure $\nu_\xi := \frac{1}{2} \nu + \frac{1}{2m}\sum_{j=1}^m \delta_{\xi_m}$, where $\delta$ denotes the Dirac measure supported at the point. It is easy to see that $L_\infty(\Omega, \nu_\xi) = L_\infty(\Omega, \nu)$. The following theorem includes Lemma~\ref{lem:s_apx}. 




\begin{theorem}\label{thm:sterm-ls}
	Let $m,s,N\in \NN$ and $s\leq N/2$ and $1\leq p \leq \infty$. Let $\D_N \subset \C(\Omega) $ be a dictionary of $N$ elements. 
	Then for any function $f\in \C(\Omega)$, we have
	\begin{align*}
		\|f-f^s_{p,m}|\xi\|_{p,m} &\leq  2^{1/p} \sigma_s(f, \D_N)_{L_p(\Omega, \nu_{\xi})}, \\
		\|f-f^s_{p,m}|\xi\|_{p,m} &\leq  2^{1/p} \sigma_s(f, \D_N)_{L_\infty(\Omega)}.
	\end{align*}
	In addition, if $\xi$ provides one sided universal discretization property for $\Sigma_{2s}(\D_N)$ (Assumption~\ref{ass:sparsecoding}), then we have
	\begin{align*}
		\|f-f^s_{p,m}\|_{L_p(\Omega,\nu)} &\leq  C_p \sigma_s(f, \D_N)_{L_p(\Omega, \nu_{\xi})}, \\
		\|f-f^s_{p,m}\|_{L_p(\Omega, \nu)} &\leq  C_p \sigma_s(f, \D_N)_{L_\infty(\Omega)},
	\end{align*}
    where
    \begin{align}\label{equ:cp}
        C_p := 2^{1/p}\left( 1+ 2(1+C_1^{-1})^{1/p} \right) + 2^{1/p}C_1^{-1/p}\left( 2 + 2(1+C_1^{-1})^{1/p} \right).
    \end{align}
\end{theorem}

\begin{proof}[Proof of Lemma~\ref{lem:s_apx} and Theorem~\ref{thm:sterm-ls}]
	Fix a linear space $V \in \X_s(\D_N)$. We write $u:=  LS_p\left(\xi, V\right)(f) $ for simplicity. It is straightforward to show that
	\begin{align}\label{eq:s-2-muxi}
		\|f|\xi\|^p_{p,m} \leq 2 \|f\|^p_{L_p(\Omega, \nu_\xi)}, \quad \forall f \in \C(\Omega).
	\end{align}
	Hence, for any $g \in V$, we have
	\begin{align*}
		\|f-g|\xi\|^p_{p,m} \leq 2 \|f-g\|^p_{L_p(\Omega, \nu_\xi)}.
	\end{align*}
	Since $u$ is the minimizer of the least square regression problem \eqref{eq:ls} w.r.t. $V$, we have for any $g\in V$
	\begin{align}\label{eq:fu}
		\|f-u|\xi\|_{p,m} \leq \|f-g|\xi\|_{p,m} \leq 2^{1/p} \|f-g\|_{L_p(\Omega, \nu_\xi)}.
	\end{align}
	Taking the minimum over all $g \in V$ gives
	\begin{align*}
		\|f-u|\xi\|_{p,m} \leq  2^{1/p} d(f, V)_{L_p(\Omega, \nu_{\xi})}.
	\end{align*}
	Since $u$ depends on the linear space $V$, by minimizing both sides for any $V\in \Sigma_s(\D_N)$ and using \eqref{eq:optimal_subspace}, we can further get
	\begin{align}\label{eqref:conclusion1}
		\|f-f^s_{p,m}|\xi\|_{p,m} \leq  2^{1/p} \sigma_s(f, \D_N)_{L_p(\Omega, \nu_{\xi})}.
	\end{align}
	Combining \eqref{eq:fu} with the following inequality
    \begin{align}\label{eq:LpLinfty}
        \|f-g\|_{L_p(\Omega, \nu_\xi)} \leq  \|f-g\|_{L_\infty(\Omega,\nu)},
    \end{align}
    we obtain an upper bound given in $L_\infty$ norm
	\begin{align}\label{eq:Linftycase}
		\|f-f^s_{p,m}|\xi\|_{p,m} \leq  2^{1/p} \sigma_s(f, \D_N)_{L_\infty(\Omega, \nu)},
	\end{align}
    with similar steps.
	
	For any $f \in \C(\Omega)$, its $L_p(\Omega, \nu)$ norm can be bounded by its $L_p(\Omega, \nu_\xi)$ norm
	\begin{align}\label{eq:mu-2-muxi}
		\|f\|_{L_p(\Omega, \nu)} &\leq 2^{1/p} \|f\|_{L_p(\Omega, \nu_\xi)} .
	\end{align}
    Given any $g_1, g_2 \in \Sigma_s(\D_N)$, Assumption~\ref{ass:sparsecoding} implies that
    \begin{align}\label{eq:muxi-2-s}
    \begin{aligned}
        \|g_1-g_2\|_{L_p(\Omega,\nu_\xi)}^p 
        &= \frac{1}{2}\|g_1-g_2\|_{L_p(\Omega,\nu)}^p +\frac{1}{2} \|g_1-g_2|\xi\|_{p,m}^p \\
        &\leq  \frac{1}{2}(1+C_1^{-1}) \|g_1-g_2|\xi\|_{p,m}^p ,
    \end{aligned}
    \end{align}
    since $g_1-g_2 \in \Sigma_{2s}(\D_N)$.
    Now we are able to conclude the following upper bound of $\|f- u \|_{L_p(\Omega, \nu_\xi)}$ by choosing an arbitrary $g \in V$
    \begin{align}\label{eq:f-u}
    \begin{aligned}
        \|f- u \|_{L_p(\Omega, \nu_\xi)} 
		&\leq \|f-g\|_{L_p(\Omega, \nu_\xi)} + \|g-u\|_{L_p(\Omega, \nu_\xi)} \\
		&\leq \|f-g\|_{L_p(\Omega, \nu_\xi)} + 2^{-1/p}(1+C_1^{-1})^{1/p} \|g-u|\xi\|_{p,m} \\
		&\leq \|f-g\|_{L_p(\Omega, \nu_\xi)} + 2^{-1/p}(1+C_1^{-1})^{1/p} \left( \|g-f|\xi\|_{p,m} +  \|f-u|\xi\|_{p,m} \right) \\
		&\leq \|f-g\|_{L_p(\Omega, \nu_\xi)} + 2^{1-1/p}(1+C_1^{-1})^{1/p} \|f-g|\xi\|_{p,m} \\
        &\leq \left( 1 + 2(1+C_1^{-1})^{1/p} \right) \|f-g\|_{L_p(\Omega, \nu_\xi)},
    \end{aligned}
	\end{align}
    where we use \eqref{eq:muxi-2-s} in the second step, and the fact that $u$ is the minimizer over $V$ in the fourth step. In the last step, we use \eqref{eq:s-2-muxi}.
    
    By choosing $g \in V $ as the minimizer of $\|f-g\|_{L_p(\Omega, \nu_\xi)}$, we obtain
    \begin{align}\label{eq:best-on-V-muxi}
        \|f- u \|_{L_p(\Omega, \nu_\xi)} 
		\leq \left( 1 + 2(1+C_1^{-1})^{1/p} \right)  d(f, V)_{L_p(\Omega, \nu_{\xi})} .
    \end{align}
    Let $f^*$ be the best estimator of $f$ among all $u=  LS_p\left(\xi, V\right)(f) $ w.r.t. $V \in \X_s(\D_N)$ under $L_p(\Omega, \nu_\xi)$ norm:
    \begin{align*}
        f^* = \argmin_{V \in \X_s(\D_N)} \|f-u\|_{L_p(\Omega, \nu_\xi)}.
    \end{align*}
    Minimizing both sides of \eqref{eq:best-on-V-muxi} for any $V \in \X_s(\D_N)$ leads to the error bound between $f$ and $f^*$
    \begin{align}\label{eq:f-ustar}
        \|f- f^* \|_{L_p(\Omega, \nu_\xi)} 
		\leq \left( 1 + 2(1+C_1^{-1})^{1/p} \right) \sigma_s(f, \D_N)_{L_p(\Omega, \nu_{\xi})} .
    \end{align}
    Since $f^*- f^s_{p,m} \in \Sigma_{2s}(\D_N)$, we can characterize their distance using the one-sided universal discretization
    \begin{align}\label{eq:ustar-2-uxi}
    \begin{aligned}
        \|f^* - f^s_{p,m} \|_{L_p(\Omega, \nu)} 
        &\leq C_1^{-1/p } \|f^*-f^s_{p,m} | \xi\|_{p,m} \\
        &\leq  C_1^{-1/p } \|f- f^*|\xi\|_{p,m} +  C_1^{-1/p } \|f-f^s_{p,m}|\xi\|_{p,m }\\
        &\leq  2^{1/p}C_1^{-1/p } \|f- f^*\|_{L_p(\Omega, \nu_\xi)} +  C_1^{-1/p } \|f-f^s_{p,m}|\xi\|_{p,m} \\
        &\leq  2^{1/p}C_1^{-1/p} \left( 2 + 2(1+C_1^{-1})^{1/p} \right) \sigma_s(f, \D_N)_{L_p(\Omega, \nu_{\xi})},
    \end{aligned}
    \end{align}
    where in the third step we use \eqref{eq:s-2-muxi}, and the last step follows from \eqref{eqref:conclusion1} and \eqref{eq:f-ustar}.
    
    Finally, following \eqref{eq:mu-2-muxi}, \eqref{eq:f-ustar}, and \eqref{eq:ustar-2-uxi}, we conclude the last statement
    \begin{align}\label{eq:finalstep}
    \begin{aligned}
        \|f-f^s_{p,m}\|_{L_p(\Omega,\nu)} 
        &\leq \|f-f^*\|_{L_p(\Omega,\nu)} + \|f^*-f^s_{p,m}\|_{L_p(\Omega,\nu)} \\
        &\leq 2^{1/p } \|f-f^*\|_{L_p(\Omega,\nu_\xi)} + \|f^*-f^s_{p,m}\|_{L_p(\Omega,\nu)} \\
        &\leq C_p \sigma_s(f, \D_N)_{L_p(\Omega, \nu_{\xi})} .
    \end{aligned}
    \end{align}
    Again, since \eqref{eq:LpLinfty} and \eqref{eq:Linftycase} holds, the keys steps can be adapted to $L_\infty$ case for some constants $C_p', C_p''$
    \begin{itemize}
        \item The inequality \eqref{eq:f-u} becomes $\|f- u \|_{L_p(\Omega, \nu_\xi)} \leq C_p' \|f-g\|_{L_\infty(\Omega, \nu)} $ .
        \item The inequality \eqref{eq:f-ustar} becomes $\|f- f^* \|_{L_p(\Omega, \nu_\xi)} \leq C_p' \sigma_s(f, \D_N)_{L_\infty(\Omega, \nu)} $.
        \item The inequality \eqref{eq:ustar-2-uxi} becomes $\|f^* - f^s_{p,m} \|_{L_p(\Omega, \nu)}  \leq \sigma_s(f, \D_N)_{L_\infty(\Omega, \nu)}$.
        \item Similarly and consequently, the final step \eqref{eq:finalstep} becomes $$\|f-f^s_{p,m}\|_{L_p(\Omega,\nu)} \leq C_p \sigma_s(f, \D_N)_{L_\infty(\Omega, \nu)}.$$
    \end{itemize}
    The proof is finished.
\end{proof}
Theorem~\ref{thm:sterm-ls} differs from the result in \cite{dai2023universal} in that their estimator is based on the $L_p$ norm in \eqref{eq:optimal_subspace}, rather than on a discretized summation.
This distinction approach makes our procedure more practical for computing the estimator and builds a close connection between sparse approximation and sparse coding problems.

\section{Proofs of Section~\ref{sec:app_general}}
\subsection{Proof of Theorem~\ref{thm:nn-estimator}}\label{subsec:cnn-rcvy}
Building on Lemma~\ref{lem:s_apx} or Theorem~\ref{thm:sterm-ls}, we are now prepared to construct CNNs that act as sparse approximators.
The next lemma was essentially demonstrated in \cite{li2024convergence}, relying on the methods developed in \cite{chen2018theoretical}. We are reconsidering the proof, as constants in error bounds are essential to subsequent analyzes.
\begin{lemma}\label{lem:cnnsparse}
	Consider a linear inverse problem $y=Ax^* + \varepsilon$ where columns of $A \in \RR^{m\times N}$ are normalized and $\|x^*\|_0\leq s$, $\|x^*\|_\infty \leq B$, and $\|\varepsilon\|_1\leq \delta$. We assume that the sparsity satisfies $s< \left( 1+ 1/\mu(\tilde{D}^m_N) \right)/2$.
    Then there exists a CNN $\phi \in \nn_{\operatorname{CNN}}(O(J\log_k(m+N)),J(m+N)^2) $ with kernel size $k$ such that for any $\|x^*\|_0\leq s$, $\|x^*\|_\infty \leq B$, and $\|\varepsilon\|_1\leq \delta$
	\begin{align*}
    \begin{gathered}
        \left\| x^* - \phi( y ) \right\|_{1} \leq   c_1 e^{-c_2J} + c_3 \delta  , \\
        \Supp \phi(y)  \subset \Supp x^*,
    \end{gathered}
	\end{align*}
    where $c_1 = sB$, $c_2 = -\ln (2\mu s-\mu) >0 $, $c_3 = 2 s \sum_{i=0}^{J} (2\mu s-\mu)^i $, and $\Supp z := \{i: z_i \neq 0\}$ for any vector $z$.
\end{lemma}
\begin{proof}
    In this proof, for the sake of self-containedness, we repeat and adjust the approach of \cite{chen2018theoretical} to obtain the convergence rate for an iteration below and then adopt an argument similar to that in \cite{li2024convergence} to convert this iterative scheme into a CNN. 
    We denote $\mu:= \mu(\tilde{D}^m_N)$, $C_A :=\max_{ij} |A_{ij}|$, and
    \begin{align*}
        \XX(s,B,\delta) :=\left \{ (x^*,\varepsilon) : \|x^*\|_0\leq s, \|x^*\|_\infty \leq B, \|\varepsilon\|_1\leq \delta \right \},
    \end{align*}
    for simplicity. It is easy to see that $C_A \leq 1$ since the columns of $A \in \RR^{m\times N}$ are $\ell_2$-normalized.
    
      We first prove that the following iteration $x^{(k)}$ converges to $x^*$
    \begin{align}\label{LISTA-CP}
    	\begin{aligned}
    		 x^{(0)} &:= 0, \\
    		 x^{(k+1)} &:= \t_{\theta^{(k)}}( x^{(k)} + A^\top ( y -  A x^{(k)})), k = 1,2,3,\dots,
    	\end{aligned}
    \end{align}
    where the soft thresholding function $\t_{\alpha}$ is defined component-wise as
    \begin{align*}
    	\t_\alpha( x)_i = \operatorname{sign}(x_i)(|x_i|-\alpha)_{+},
    \end{align*}
     and the threshold $\theta^{(k)}$ is given by
     \begin{align}
         \theta^{(k)} := \mu \sup_{(x^*,\varepsilon) \in \XX(s,B,\delta) } \{  \| x^{(k)}(x^*,\varepsilon) - x^* \|_1 \} + C_A \delta.
     \end{align}

 Denote $S := \Supp x^*$ ($|S| \leq s$) and $A_i$ as the $i$-th column of $A$. We assume that $x^{(k)}_i = 0$, for any $i\notin S$. Then we can rewrite $x^{(k+1)}_i$ as
     \begin{align*}
         x^{(k+1)}_i 
         &= \t_{\theta^{(k)}} ( x^{(k)}_i + A^\top_i A ( x^* - x^{(k)}  ) + A^\top_i \varepsilon ), \\
         &= \t_{\theta^{(k)}} ( \sum_{j\in S} A^\top_i A_j ( x^*_j - x^{(k)}_j  ) + A^\top_i \varepsilon ).
     \end{align*}
     Since the definition of $\theta^{(k)}$ implies that
     \begin{align*}
         \left | \sum_{j\in S} A^\top_i A_j ( x^*_j - x^{(k)}_j  ) + A^\top_i \varepsilon \right| \leq   \mu \| x^{(k)} - x^* \|_1 + C_A \delta \leq \theta^{(k)} ,
     \end{align*}
     we obtain $x^{(k+1)}_i = 0 $ for any $i\notin S$. As $x^{(0)} = 0$, we can conclude that $x^{(k)}_i = 0$, $\forall i \notin S$ is true for any $k$. This means that $\Supp x^{(k)} \subset \Supp x^*$.

     Define
     \begin{align*}
         \partial \ell_1(z)_i := 
         \left\{
         \begin{aligned}
             {\operatorname{sign}(z_i)}, \quad z_i \neq 0, \\
             [-1,1], \quad z_i = 0.
         \end{aligned}
         \right.
     \end{align*}
     Then for any $i\in S$, 
     \begin{align*}
         \begin{aligned}
            x^{(k+1)}_i 
            &= \t_{\theta^{(k)}} ( x^{(k)}_i + \sum_{j\in S} A^\top_i A_j ( x^*_j - x^{(k)}_j  ) + A^\top_i \varepsilon ) \\
            &= \t_{\theta^{(k)}} ( x^{(k)}_i + \sum_{j\in S, j\neq i} A^\top_i A_j ( x^*_j - x^{(k)}_j  ) + A_i^\top A_i (x^*_i - x^{(k)}_i) + A^\top_i \varepsilon ) \\
            &\in x_i^*+ \sum_{j\in S, j\neq i} A^\top_i A_j ( x^*_j - x^{(k)}_j  ) + A^\top_i \varepsilon - \theta^{(k)} \partial\ell_1(x_i^{(k+1)}) ,
         \end{aligned}
     \end{align*}
     where we use the fact that $A$ is column-wise normalized.
     This further implies that
     \begin{align}\label{eq:xk1}
         x^{(k+1)}_i - x^*_i \in \sum_{j\in S, j\neq i} A^\top_i A_j ( x^*_j - x^{(k)}_j  ) + A^\top_i \varepsilon - \theta^{(k)} \partial\ell_1(x_i^{(k+1)}).
     \end{align}
     Since $\partial\ell_1( x_i^{(k+1)})$ is bounded by $1$, from \eqref{eq:xk1} we can get the upper bound of $x^{(k+1)}_i - x^*_i$ for any $i\in S$
     \begin{align}\label{eq:xsupport}
     \begin{aligned}
         |x^{(k+1)}_i - x^*_i| 
         &\leq \left| \sum_{j\in S, j\neq i} A^\top_i A_j ( x^*_j - x^{(k)}_j  ) + A^\top_i \varepsilon \right| + \theta^{(k)} \\
         &\leq \mu \sum_{j\in S, j\neq i} | x^*_j - x^{(k)}_j  | + C_A \delta + \theta^{(k)}.
     \end{aligned}
     \end{align}
     Summing both sides of \eqref{eq:xsupport} for any $i\in [N]$ and noticing that $\Supp x^{(k)} \subset \Supp x^*$, we obtain
     \begin{align*}
         \sup_{(x^*,\varepsilon) \in \XX(s,B,\delta) } \|x^{(k+1)} - x^*\|_1 
         &= \sup_{(x^*,\varepsilon) \in \XX(s,B,\delta) } \sum_{i\in S}|x^{(k+1)}_i - x^*_i| \\
         &\leq \sup_{(x^*,\varepsilon) \in \XX(s,B,\delta) } \left\{   \mu \sum_{i\in S} \sum_{j\in S, j\neq i} | x^*_j - x^{(k)}_j  |   \right\} + sC_A \delta + s\theta^{(k)} \\
         &\leq \mu (s-1) \sup_{(x^*,\varepsilon) \in \XX(s,B,\delta) }  \left\{ \|x^{(k)} - x^*\|_1 \right \} + sC_A \delta + s\theta^{(k)} \\
         &\leq (2\mu s-\mu) \sup_{(x^*,\varepsilon) \in \XX(s,B,\delta) }  \left\{ \|x^{(k)} - x^*\|_1 \right \} +2 sC_A \delta,
     \end{align*}
     where in the second and last steps we use the definition of $\theta^{(k)}$.
     By induction, we obtain
     \begin{align*}
         &\sup_{(x^*,\varepsilon) \in \XX(s,B,\delta) } \|x^{(k+1)} - x^*\|_1 \\
         &\leq (2\mu s-\mu)^{k+1} \sup_{(x^*,\varepsilon) \in \XX(s,B,\delta) }  \left\{ \|x^{(0)} - x^*\|_1 \right \} +2 sC_A \delta \sum_{i=0}^{k} (2\mu s-\mu)^i \\
         &= c_1 e^{-c_2 (k+1) } +  \delta\left( 2 s \sum_{i=0}^{k} (2\mu s-\mu)^i \right).
     \end{align*}
     since $x^{(0)} = 0$ and $C_A \leq 1$.

     Lemma 14 in \cite{li2024convergence} implies that there exists a CNN $\phi \in \nn_{\operatorname{CNN}}(O(J\log_k(m+N)),J(m+N)^2) $ with kernel size $k$ such that $\phi(y) = x^{(J)}(x^*,\varepsilon)$ for any $(x^*,\varepsilon) \in \XX(s,B,\delta)$.
     The proof is complete.
\end{proof}

To determine the conditions of the above lemma for sparse approximation, we also need to assess how a dictionary influences the $\ell_p$ norm of sparse vectors.
The following estimation commonly utilized in sparse coding is taken from \cite{elad2010sparse}. For completeness, we provide the proof.

\begin{lemma}\label{lem:mutual}
    Given a matrix $D \in \RR^{n \times d}$ that is column-wise normalized, for any $ x\in \RR^d$ with sparsity $\| x\|_0 \leq s$, we have
    \begin{align*}
        \left(1-(s-1) \mu( D)\right)\| x\|_2^2 \leq \| D x\|_2^2 \leq \left(1+(s-1) \mu( D)\right)\| x\|_2^2 .
    \end{align*}
\end{lemma}
\begin{proof}
    Let $ Q = \bm 1 \bm 1^\top$ where $\bm 1= (1,1,\dots, 1)^\top \in \RR^d$. Given a matrix $D \in \RR^{n \times d}$ that is column-normalized, it is easy to see that for any $ x\in \RR^d$ with sparsity $\| x\|_0 \leq s$,
	\begin{align*}
		\begin{aligned}
			\| D x\|_2^2 =  x^\top  D^\top  D  x 
            &\geq  x^\top  x - \mu( D)| x|^\top(  Q-I)| x| \\
			&\geq \left(1+ \mu( D)\right)\| x\|_2^2 - \mu( D) \| x\|_1^2 \\
			&\geq \left(1+ \mu( D)\right)\| x\|_2^2 - \mu( D)s \| x\|_2^2 \\
			&= \left(1-(s-1) \mu( D)\right)\| x\|_2^2 ,
		\end{aligned}
	\end{align*}
	where $|x|:=(|x_i|)_{i=1}^d$, in the third step we use the inequality $\| c\|_1 \leq \sqrt{s}\| c\|_2$ for any $ c \in \RR^s$, and in the first step we use the following fact
	\begin{align*}
		 x^\top  D^\top  D  x = \| x \|_2^2 + \sum_{i \neq j}  D_i^\top   D_j x_i x_j \geq \| x \|_2^2 - \sum_{i \neq j} \mu( D) |x_i x_j| = \| x \|_2^2 - \mu( D)| x|^\top ( Q- I) | x|.
	\end{align*}
    
    The claim for the upper bound follows similar steps
	\begin{align*}
		\begin{aligned}
			\| D x\|_2^2 =  x^\top  D^\top  D  x 
            &\leq  x^\top  x - \mu( D)| x|^\top( I -  Q)| x| \\
			&\leq \left(1- \mu( D)\right)\| x\|_2^2 + \mu( D) \| x\|_1^2 \\
			&\leq \left(1- \mu( D)\right)\| x\|_2^2 + \mu( D)s \| x\|_2^2 \\
			&= \left(1+(s-1) \mu( D)\right)\| x\|_2^2 ,
		\end{aligned}
	\end{align*}
	where in the first step we use
	\begin{align*}
		 x^\top  D^\top  D  x = \| x \|_2^2 + \sum_{i \neq j}  D_i^\top   D_j x_i x_j \leq \| x \|_2^2 + \sum_{i \neq j} \mu( D) |x_i x_j| = \| x \|_2^2 + \mu( D)| x|^\top ( Q- I) | x|.
	\end{align*}
\end{proof}

Using Theorem~\ref{thm:sterm-ls} together with Lemma~\ref{lem:cnnsparse}, we are now prepared to establish the proof of Theorem~\ref{thm:nn-estimator}. In contrast to Theorem~\ref{thm:nn-estimator}, Theorem~\ref{thm:nn-estimator-full} also encompasses the case $\sigma_s(F, \D_N)_{L_p(\Omega,\nu_\xi)}$.

\begin{theorem}\label{thm:nn-estimator-full}
	Let $m, s, k, J, N \in \mathbb{N}$, $1 \le p < \infty$, $B > 0$, and assume $s < \bar{s}$.
    Then there exists a CNN $\Phi \in \nn_{\operatorname{CNN}}(O(J\log_k(m+N)),J(m+N)^2) $ with kernel size $k$, such that for any $f\in F := \{ f \in \C(\Omega) : \|f\|_{\C(\Omega)} \le B \}$
    \begin{enumerate}
        \item $\| \Phi(\tilde{f}^m) \|_0 \leq s$;
        \item  $\left\| w_{p,m}^s - \Phi( \tilde{f}^m ) \right\|_{p} \leq   C_4 e^{-C_5 J} + C_6 \sigma_s(F, \D_N)_{L_\infty(\Omega,\nu)}$.
    \end{enumerate}
	In addition, if the one-sided universal discretization (Assumption~\ref{ass:sparsecoding}) holds, then 
	\begin{align}
		\|f-f_s\|_{L_p(\Omega,\nu)} 
		\leq   C_7 e^{-C_5 J} + C_8 \sigma_s(F, \D_N)_{L_\infty(\Omega,\nu)}.
	\end{align}
    where $f_s:= \sum_{i=1}^N \Phi( \tilde{f}^m )_i u_i $ and the constants $C_4,\dots,C_8$ depend on $p$, $m$, $B$, $s$, $N$, and $\tilde{D}^m_N$, as given in \eqref{eq:C4}, \eqref{eq:C5}, \eqref{eq:C6}, \eqref{eq:C7}, and \eqref{eq:C8} in Appendix~\ref{subsec:cnn-rcvy}.

    The above error bounds hold also for the $\|\cdot\|_{L_p(\Omega,\nu_\xi)}$ case.
\end{theorem}

\begin{proof}[Proof of Theorem~\ref{thm:nn-estimator} and Theorem~\ref{thm:nn-estimator-full}]
	

	According to the definition \eqref{eq:s-coef}, we can represent $\tilde{f}^m$ as $\tilde{f}^m =  \tilde{D}_N^m w^s_{p,m} + \tilde{r}^m_N $ for some residual term $\tilde{r}^m_N \in \RR^m$, which satisfies
	\begin{align}\label{eq:r_p}
    \begin{aligned}
        \sup_{\|f \|_{\C(\Omega)} \leq B} \| \tilde{r}^m_N \|_1
        &\leq m^{1-1/p}  \sup_{\|f \|_{\C(\Omega)} \leq B} \| \tilde{r}^m_N \|_p \\
        &= m  \sup_{\|f \|_{\C(\Omega)} \leq B} \| \tilde{r}^m_N \|_{p,m} \\
        &\leq
		2^{1/p} m \cdot \sigma_s(F, \D_N)_{L_\infty(\Omega,\nu)} ,
    \end{aligned}
	\end{align}
    where the third step follows from the fact that $\tilde{r}^m_N = \tilde{f}^m -  \tilde{D}_N^m w^s_{p,m} = (f-f^s_{p,m}|\xi) $ and Lemma~\ref{lem:s_apx} (Theorem~\ref{thm:sterm-ls}).

    Let $L \in \RR^{N\times N}$ be a diagonal matrix with diagonal elements $\frac{1}{\sqrt{\sum_{t=1}^m \left| u_i(\xi_t)\right|^2}}$
    \begin{align}\label{eq:L}
        L :=
        \begin{bmatrix}
        \frac{1}{\sqrt{\sum_{t=1}^m \left| u_1(\xi_t)\right|^2}} & 0 & \dots & 0 \\
        0 & \frac{1}{\sqrt{\sum_{t=1}^m \left| u_2(\xi_t)\right|^2}}  & \dots & 0 \\
        \vdots & \vdots & \ddots & 0 \\
        0 & \dots & \dots & \frac{1}{\sqrt{\sum_{t=1}^m \left| u_N(\xi_t)\right|^2}}
    \end{bmatrix}
        .
    \end{align}
    Then we can rewrite the inverse problem as $\tilde{f}^m =  \tilde{D}_N^m L L^{-1} w^s_{p,m} + \tilde{r}^m_N$ with $\tilde{D}_N^m L$ being column-wise normalized. For simplicity, we denote $\bar{D}_N^m  := \tilde{D}_N^m L$ and $\bar{w}^s_{p,m} := L^{-1} w^s_{p,m}$.
    For this inverse problem, our next step is to derive the necessary conditions on $ 
    \bar{w}^s_{p,m}$ as required in Lemma~\ref{lem:cnnsparse}.
    We use Lemma~\ref{lem:mutual} to estimate the $\ell_\infty$ norm of $\bar w^s_{p,m}$
    \begin{align}\label{eq:wpupper}
    \begin{aligned}
         \|\bar w^s_{p,m}\|_\infty &\leq \|\bar w^s_{p,m}\|_2 \leq \frac{\left\|\tilde{D}_N^m w^s_{p,m}  \right\|_2}{\sqrt{1-(s-1) \mu(\tilde{D}_N^m )}} \leq \frac{\left\| \tilde{f}^m \right\|_2 + \left\| \tilde r^m_N \right\|_2 }{\sqrt{1-(s-1) \mu(\tilde{D}_N^m )}} \\
        &\leq \frac{B \sqrt{m} + 2^{1/p} B m  }{\sqrt{1-(s-1) \mu(\tilde{D}_N^m )}} ,
    \end{aligned}
    \end{align}
    where in the first step, it is straightforward from the definition that $\mu(\tilde{D}_N^m ) = \mu(\bar{D}_N^m )$, and in the last step we use $\|c\|_2 \leq \sqrt{m} \|c\|_\infty$ and $\left\| \tilde{f}^m \right\|_\infty \leq B$ to control $\left\| \tilde{f}^m \right\|_2$, and $\|c\|_2 \leq \|c\|_1$, \eqref{eq:r_p}, and the fact that $\sigma_s(f, \D_N)_{L_\infty(\Omega,\nu)} \leq \|f\|_{L_\infty(\Omega,\nu)} \leq B$ to control $\left\| \tilde r^m_N \right\|_2$.
    
    Since
    \begin{align*}
        1-(s-1)\mu(\tilde{D}_N^m ) > 1- \left(\frac{1}{2}\left(1+\frac{1}{\mu(\tilde{D}_N^m )}\right)-1 \right)\mu(\tilde{D}_N^m ) = \frac{1}{2} + \frac{1}{2}\mu(\tilde{D}_N^m ) \geq \frac{1}{2},
    \end{align*}
    we can further simplify \eqref{eq:wpupper} as
    \begin{align}\label{eq:w_upper}
        \|\bar w^s_{p,m}\|_\infty \leq \frac{B \sqrt{m} + 2^{1/p} B m  }{\sqrt{1-(s-1) \mu(\tilde{D}_N^m )}} \leq B \sqrt{2m} + 2^{1/2+1/p} B m \leq 6 Bm.
    \end{align}
    Following \eqref{eq:r_p}, \eqref{eq:wpupper}, the pair $(\bar w^s_{p,m}, \tilde{r}^m_N)$ in $\tilde{f}^m =  \bar{D}_N^m \bar w^s_{p,m} + \tilde{r}^m_N $ satisfies
    \begin{align*}
        \|\bar w^s_{p,m}\|_0\leq s, \quad \|\bar w^s_{p,m}\|_\infty\leq 6Bm, \quad \|\tilde{r}^m_N\|_1\leq 2^{1/p} m \cdot \sigma_s(F, \D_N)_{L_\infty(\Omega,\nu)}.
    \end{align*}
    
	  Lemma~\ref{lem:cnnsparse} shows the following convergence rate of using CNNs to approximate the sparse coefficient vector $\bar w_{p,m}^s$. 
    When $s<\frac{1}{2} \left(1 + 1/\mu\left(\tilde{D}_N^m\right) \right)$, we can learn the sparse coefficient vector $\bar w^{s}_{p,m}$ through a CNN $\phi$ with kernel size $k$, $O(J \log_k (m+N) )$ layers and at most $O(J(N+m)^2)$ nonzero parameters such that for any $\|f\|_{\C(\Omega)} \leq B $, we have
    \begin{align}\label{eq:nn2discrete-ell2}
    \begin{aligned}
         \left\|  \phi(\tilde{f}^m) -\bar w^s_{p,m}  \right\|_1 
        &\leq 6 Bms\cdot e^{-cJ} + C 2^{1/p} m \cdot\sigma_s(F, \D_N)_{L_\infty(\Omega,\nu)},
    \end{aligned}
	\end{align}
    where $c = -\ln (2\mu s-\mu) >0 $, $C = 2 s \sum_{i=0}^{J} (2\mu s-\mu)^i $ with $\mu:=\mu( \tilde{D}_N^m )$. This implies
    \begin{align*}
    \begin{aligned}
        \left\| L \phi(\tilde{f}^m) - w^s_{p,m}  \right\|_1 
        &=
         \left\| L \phi(\tilde{f}^m) - L \bar w^s_{p,m}  \right\|_1 \\
         &\leq \left( \max_{i\in[N]} \frac{1}{\sqrt{\sum_{t=1}^m \left| u_i(\xi_t)\right|^2}} \right) \left\|  \phi(\tilde{f}^m) -\bar w^s_{p,m}  \right\|_1  \\
        &\leq 6 c_0 Bms\cdot e^{-cJ} + c_0 C 2^{1/p} m \cdot\sigma_s(F, \D_N)_{L_\infty(\Omega,\nu)},
    \end{aligned}
	\end{align*}
    where we denote $c_0 := \max_{i\in[N]} \frac{1}{\sqrt{\sum_{t=1}^m \left| u_i(\xi_t)\right|^2}}$. 
    According to \cite{li2024approximation}, the linear map $L$ can be performed by a CNN of depth $O(\log_k N)$ with at most $O(N^2)$ parameters. Therefore, by denoting a CNN as $\psi(\cdot):= L \cdot$, the composition $\Phi :=\psi \circ \phi$ has $O(J \log_k (m+N) )$ layers and at most $O(J(N+m)^2)$ nonzero parameters.

    Accordingly, we obtain the following error bound when using a CNN to approximate the sparse coefficients
	\begin{align}\label{eq:nn2discrete}
    \begin{aligned}
         \left\|\Phi(\tilde{f}^m) - w^s_{p,m}  \right\|_p 
        &\leq N^{1-1/p} \left\| L \phi(\tilde{f}^m) - w^s_{p,m}  \right\|_1 \\
        &\leq C_4 e^{-C_5J} + C_6 \sigma_s(F, \D_N)_{L_\infty(\Omega,\nu)},
    \end{aligned}
	\end{align}
    with
    \begin{align}
        C_4 &:= 6 c_0 B m s N^{1-1/p},  \label{eq:C4} \\ 
        C_5 &:= -\ln (2\mu(\tilde{D}^m_N) s-\mu(\tilde{D}^m_N)), \label{eq:C5} \\
        C_6 &:= c_0 C 2^{1/p} m N^{1-1/p}. \label{eq:C6} 
    \end{align}
    Moreover, the support of $\Phi(\tilde{f}^m)$ is included in the support of $ w^s_{p,m}$, and we have $\|\Phi(\tilde{f}^m)\|_0 \leq s$.

	If we denote the $s$-sparse estimator of $f$ with coefficients learned by CNN $\Phi$ as $f_s:= \sum_{i=1}^N \Phi(\tilde{f}^m)_i u_i $, then we also have
	\begin{align}\label{eq:coeff-dec}
    \begin{aligned}
        &\|f-f_s\|_{L_p(\Omega,\nu)} \\
		&\leq  \|f-f^s_{p,m}\|_{L_p(\Omega,\nu)} + \|f^s_{p,m}-f_s\|_{L_p(\Omega,\nu)} .
    \end{aligned}
	\end{align}
    Let $C_{m,p}:= \max\{1, m^{1/p-1/2} \}$. 
    Lemma~\ref{lem:mutual} and $s < \frac{1}{2}\left(1 + \frac{1}{\mu\left(\tilde{D}_N^m\right)} \right)$ imply
    \begin{align}\label{eq:westimates}
    \begin{aligned}
        \left  \| \tilde{D}_N^m w^s_{p,m} - \tilde{D}_N^m \Phi(\tilde{f}^m) \right  \|_{p} 
        &= \left  \| \bar{D}_N^m \bar w^s_{p,m} - \bar{D}_N^m \phi(\tilde{f}^m) \right  \|_{p}  \\
        &\leq  C_{m,p} \left  \| \bar{D}_N^m \bar w^s_{p,m} - \bar{D}_N^m \phi(\tilde{f}^m) \right  \|_{2} \\
        &\leq C_{m,p} \sqrt{1+(2s-1)\mu\left(\tilde{D}_N^m\right)} \left  \| \bar w^s_{p,m} -  \phi(\tilde{f}^m) \right  \|_{2} \\
        &\leq 2 C_{m,p}  \left  \| \bar w^s_{p,m} -  \phi(\tilde{f}^m) \right  \|_{2}.
    \end{aligned}
    \end{align}
    
    Since $f^s_{p,m}- f_s \in \Sigma_{2s} (\D_N)$, we can use the one-sided universal discretization property and derive the inequality of the second term of the upper bound in \eqref{eq:coeff-dec}
    \begin{align}\label{eq:cnn-part}
        \begin{aligned}
            &\|f^s_{p,m}-f_s\|_{L_p(\Omega,\nu)} \\
            &\leq C_1^{-1/p} \left \| f^s_{p,m}-f_s | \xi   \right \|_{p,m} \\
            &= C_1^{-1/p} \left  \| \tilde{D}_N^m w^s_{p,m} - \tilde{D}_N^m \Phi(\tilde{f}^m) \right  \|_{p,m} \\
            &\leq 2 C_1^{-1/p} C_{m,p} m^{-1/p} \left(6 Bms \cdot e^{-cJ} + C 2^{1/p} m \cdot \sigma_s(F, \D_N)_{L_\infty(\Omega,\nu)} \right),
        \end{aligned}
    \end{align}
    where in the first step we use one sided universal discretization property Assumption~\ref{ass:sparsecoding} and in the last step we use \eqref{eq:nn2discrete-ell2}, $\|\cdot\|_2 \leq \|\cdot\|_1$, and \eqref{eq:westimates}.
    
    Substituting Theorem~\ref{thm:sterm-ls} and \eqref{eq:cnn-part} into \eqref{eq:coeff-dec}, we finally conclude
    \begin{align*}
        \|f-f_s\|_{L_p(\Omega,\nu)} 
        &\leq C_7 e^{-C_5 J}  + C_8 \sigma_s(F, \D_N)_{L_\infty(\Omega)},
    \end{align*}
    where
    \begin{align}
        C_7 &:=  12 C_1^{-1/p} C_{m,p} m^{1-1/p} Bs, \label{eq:C7} \\
        C_8 &:=  2^{1+1/p} C C_1^{-1/p} C_{m,p} m^{1-1/p} +C_3 , \label{eq:C8} 
    \end{align}
    where $C = 2 s \sum_{i=0}^{J} (2\mu(\tilde{D}^m_N) s-\mu(\tilde{D}^m_N))^i $ and $C_3$ depends only on $p$.

    Applying Theorem~\ref{thm:sterm-ls} and proceeding analogously, the error bounds can be straightforwardly extended to the $\|\cdot\|_{L_p(\Omega,\nu_\xi)}$ setting.
\end{proof}

\subsection{Proof of Theorem~\ref{thm:holder}} \label{sec:holder}
To establish our main results, we begin by introducing some notation. The modulus of continuity of a continuous function $h$ is defined as
\begin{align*}
    \omega_{h}(r;\Omega) : =  \sup \left\{ |h(x_1)-h(x_2)|: \|x_1-x_2\|_{2} \leq r , x_1,x_2 \in \Omega \right\}.
\end{align*}

Furthermore, we abuse the notation $\D_N$ for a dictionary as a linear transform $\D_N : \RR^N \rightarrow C(\Omega)$, defined as the linear combination of elements in the dictionary with respect to a coefficient vector $w$: $\D_N (w) = \sum_{i=1}^N w_i u_i $. 

Let $\hat{\P} := \P\circ \D_N : \RR^N \rightarrow \RR$ and $A_s := \{w\in\RR^N: \|w\|_0 \leq s \}$. Then we shall prove that $\omega_{\P}$ is almost the same value as $\omega_{\hat{\P}}$.

\begin{lemma}\label{lem:modulus_equiv}
    If $\xi \subset \RR^m$ provides the universal discretization property of $\Sigma_{2s}(\D_N)$, then for any nonlinear functional $\P$,  it holds
    \begin{itemize}
        \item[(1)] $\omega_{\hat{\P}}(r;A_s ) \leq \omega_{\P}\left(\tilde c_1 r ; \Sigma_{s} (\D_N) \right)_p$, where $\tilde c_1$ is explicitly given in the proof.

        \item[(2)] $\omega_{\hat{\P}}(r;A_s ) \geq \omega_{\P}\left(\tilde c_2 r ; \Sigma_{s} (\D_N) \right)_p$, where $\tilde c_2$ is explicitly given in the proof.
    \end{itemize}
\end{lemma}
\begin{proof}
    Let $c_1 := \sqrt{1+(2s-1) \mu( \tilde{D}_N^m)}$ and $c_2 := \max\{1, m^{1/p-1/2}\}$. For any $g_1 = \D_N(w_1)$ and $g_2 = \D_N(w_2)$ with $\|w_1\|_0\leq s$ and $\|w_2\|_0\leq s$, we can characterize the upper bound of $\|g_1-g_2\|_{L_p(\Omega,\nu)}$ by $\|w_1-w_2\|_2$
    \begin{align*}
        \|g_1-g_2\|_{L_p(\Omega,\nu)} 
        &\leq C_1^{-1/p}m^{-1/p} \|\tilde{D}_N^m (w_1-w_2)\|_p \\
        &\leq c_2 C_1^{-1/p}m^{-1/p} \|\tilde{D}_N^m L L^{-1} (w_1-w_2)\|_2 \\
        &\leq c_1 c_2 C_1^{-1/p}m^{-1/p} \| L^{-1} (w_1-w_2) \|_2 \\
        &\leq c_1 c_2 C_1^{-1/p}m^{-1/p} \|L^{-1}\|_2 \| w_1-w_2 \|_2,
    \end{align*}
    where $L$ is given in \eqref{eq:L} and used to normalize columns of $\tilde{D}_N^m$, in the first step we use Assumption~\ref{ass:sparsecoding} and in the third step we use Lemma~\ref{lem:mutual} and $\mu( \tilde{D}_N^m) = \mu( \tilde{D}_N^m L)$.
    
     This implies that
    \begin{align*}
        \omega_{\hat{\P}}(r;A_s )
        &= \sup \{ |\hat{\P}(w_1)-\hat{\P}(w_2)|: \|w_1-w_2\|_2\leq r, w_1, w_2 \in A_s \} \\
        &\leq \sup \{ |\P(g_1)-\P(g_2)|: \|g_1-g_2\|_{L_p(\Omega,\nu)}\leq c_1c_2 C_1^{-1/p}  m^{-1/p}\|L^{-1}\|_2 r, g_1, g_2 \in \Sigma_s (\D_N)  \} \\
        &= \omega_{\P}\left( c_1c_2C_1^{-1/p}  m^{-1/p}\|L^{-1}\|_2 r ; \Sigma_s (\D_N) \right)_p .
    \end{align*}
    Hence, the desired bound in $(1)$ holds with $\tilde c_1$ defined as
    \begin{align}\label{eq:c1tilde}
        \tilde c_1 = C_1^{-1/p}  m^{-1/p} \|L^{-1}\|_2 \max\{1, m^{1/p-1/2}\} \sqrt{1+(2s-1) \mu( \tilde{D}_N^m)}.
    \end{align}
    
    Denote $c_3 =1/ \sqrt{1-(2s-1)\mu\left( \tilde{D}_N^m \right)}$ and $c_4:= \max\{1, m^{1/2-1/p}\} $. Conversely and similarly, we can also use Assumption~\ref{ass:sparsecoding} and Lemma~\ref{lem:mutual} to show that
    \begin{align*}
        \| w_1-w_2\|_2 
        &\leq \|L\|_2 \|L^{-1}( w_1-w_2)\|_2 \\
        &\leq c_3 \|L\|_2 \|\tilde{D}_N^m L L^{-1} (w_1-w_2)\|_2 \\
        &\leq  c_4 c_3 \|L\|_2 \|\tilde{D}_N^m (w_1-w_2)\|_p \\
        &\leq m^{1/p} C_2^{1/p} c_3 c_4 \|L\|_2 \|g_1-g_2\|_{L_p(\Omega,\mu)}.
    \end{align*}
    This implies
    \begin{align*}
        \omega_{\hat{\P}}(r;A_s )
        &= \sup \{ |\hat{\P}(w_1)-\hat{\P}(w_2)|: \|w_1-w_2\|_2\leq r, w_1, w_2 \in A_s \} \\
        &\geq \sup \{ |\P(g_1)-\P(g_2)|: \|g_1-g_2\|_{L_p(\Omega,\mu)}\leq r/(m^{1/p} C_2^{1/p} c_3 c_4 \|L\|_2), g_1, g_2 \in \Sigma_s (\D_N)  \} \\
        &= \omega_{\P}\left(  r/(m^{1/p} C_2^{1/p} c_3 c_4\|L\|_2) ; \Sigma_s (\D_N) \right)_p .
    \end{align*}
    Thus, the bound stated in $(2)$ also holds with $\tilde c_2$ given by
    \begin{align}\label{eq:c2tilde}
        \tilde c_2 := \frac{\sqrt{1-(2s-1)\mu\left( \tilde{D}_N^m \right)}}{m^{1/p} C_2^{1/p} \|L\|_2 \max\{1, m^{1/2-1/p}\}} .
    \end{align}
    We complete the proof.
\end{proof}

To prove Theorem~\ref{thm:holder} for functionals, we decompose the error into two parts: one arising from a sparse approximation and the other from approximating a continuous nonlinear mapping. Observe that once the sparse representation has been learned via a CNN encoder, the input fed into the DNN component of the functional neural network becomes sparse.
This means that by fully leveraging the sparsity of the input, we anticipate an enhancement in the neural network's complexity. In what follows, we will discuss H\"{o}lder continuous functions and demonstrate how the bound on the second component improves. The subsequent theorem was principally proved in \cite{yarotsky2017error}, and we modify the approach for sparse inputs to enhance our results.

\begin{lemma}\label{thm:yarotsky}
    Let $d, r\in \NN$ and $\beta \in (0,1]$. Let $f \in C^{r,\beta}([-\frac{1}{2},\frac{1}{2}]^d) $ with $\|f\|_{C^{r,\beta}([-\frac{1}{2}, \frac{1}{2}]^d)} \leq 1$. For any $M >0$, there exists a ReLU neural network $\phi$ with depth $O(\ln M)$ and at most $O(M)$ nonzero parameters such that $\| f - \phi\|_{L_\infty \left(A_s\cap [-\frac{1}{2}, \frac{1}{2}]^d \right)} \leq (M/\ln M)^{-(r+\beta)/s}$.
\end{lemma}
\begin{proof}
Define
\begin{align*}
\begin{aligned}
    h\left(x; \frac{m}{N} \right) 
    :
    =
    \begin{cases}
        1,  & x \in \left[-\frac{1}{2}+\frac{m}{N}-\frac{1}{3N}, -\frac{1}{2}+ \frac{m}{N}+\frac{1}{3N} \right] , \\
        3N\left(x-(-\frac{1}{2}+\frac{m}{N}-\frac{2}{3N})\right),  & x \in \left[-\frac{1}{2}+\frac{m}{N}-\frac{2}{3N}, -\frac{1}{2}+ \frac{m}{N}-\frac{1}{3N} \right] , \\
        3N\left(-x+(-\frac{1}{2}+\frac{m}{N}+\frac{2}{3N})\right),  & x \in \left[-\frac{1}{2}+\frac{m}{N}+\frac{1}{3N}, -\frac{1}{2}+ \frac{m}{N}+\frac{2}{3N} \right] , \\
        0, &x  \notin \left[ -\frac{1}{2}+ \frac{m}{N}-\frac{2}{3N}, -\frac{1}{2}+ \frac{m}{N}+\frac{2}{3N}\right].
    \end{cases}
\end{aligned}
\end{align*}
Summing over all $m \in [N]_0$ where $[N]_0 := \{0,1,\dots,N\}$, we obtain the following partition of unity
\begin{align*}
    \sum_{m=0}^N h\left(x; \frac{m}{N} \right) = 1,\quad \forall x \in \left [-\frac{1}{2},\frac{1}{2} \right].
\end{align*}
Similarly, we have the following extended partition of unity on $[-\frac{1}{2},\frac{1}{2}]^d$
\begin{align*}
    \sum_{ \bm m \in [N]_0^d} h_{ \bm m}( x) = 1, \quad \forall x \in \left[-\frac{1}{2},\frac{1}{2} \right]^d,
\end{align*}
where each $h_{ \bm m}( x)$ has a support with a center point $(-\frac{1}{2}+\frac{m_i}{N})_{i=1}^d$,
\begin{align*}
    h_{\bm m}( x) = \prod_{i=1}^d h\left(x_i; \frac{m_i}{N} \right).
\end{align*}
The above partition of unity has several properties, namely:
\begin{itemize}
    \item[(a)] $\|h_{\bm m}\|_{L_\infty([0,1]^d)} = 1$, \\
    \item[(b)]  $\Supp h_{\bm m}\subset \prod_{i =1}^d \left[-\frac{1}{2}+ \frac{m_i}{N}-\frac{1}{N},-\frac{1}{2}+\frac{m_i}{N}+\frac{1}{N}\right]$.
    \item[(c)] there exists at most $2^d (N+1)^s \binom{d}{s}$ choice of $\bm m \in [N]_0^d$ such that $h_{\bm m}(x) \not\equiv 0$ for any $ x \in A_s$.
\end{itemize}
The claim $\text{(c)}$ can easily be seen from the calculation that: $(i)$ there are $\binom{d}{s}$ elements in $\{ \#\Gamma =s :\Gamma \subset [N]\}$; $(ii)$ fixing a $\Gamma$, the corresponding hyperplane $\{x: x_i=0, i \notin \Gamma\}$ has an intersection with the supports of at most $(N+1)^s$ functions $h_{\bm m}(x)$ with center points $(-\frac{1}{2}+\frac{m_i}{N})_{i=1}^d$; $(iii)$ since for each $x_i$, there are only one $m_i$ such that at most $h(x_i,\frac{m_i}{N})$ and $h(x_i,\frac{m_{i+1}}{N})$ are nonzero (i.e., $x_i$ belongs to their intersection), and hence there are at most $2^d$ functions such that $h_{\bm m}(x)\neq 0$ for each fixed $x\in A_s$.

According to property (c), we derive
\begin{align*}
    \sum_{ \bm m \in \Lambda} h_{ \bm m}(\bm x) = 1, \quad \forall x \in A_s,
\end{align*}
where we denote $\Lambda \subset [N]_0^d$ as the collection of $\bm m$, satisfying $h_{\bm m}(x)  \not\equiv 0$ for any $x \in A_s$. 

Using \cite[Lemma A.8]{petersen2018optimal}, we can find a Taylor polynomial 
\begin{align*}
    P_{\bm m}( x) = \sum_{\|\bm n\|_1 \leq r} \frac{\partial^{\bm n} f( x_{\bm m})}{\bm n!} ( x-x_{\bm m})^{\bm n},
\end{align*}
that approximates $f$ at $ x_{\bm m} = (-1/2+ m_1/N, \dots, -1/2 + m_d/N)$ within the error
\begin{align}\label{eq:localapproximation}
    |f(x) - P_{\bm m}(x)| \leq \frac{d^r}{r!} \|x-x_{\bm m}\|_2^{r+\beta},
\end{align}
where $\bm n ! = \prod_{i=1}^d n_i !$ and $( x- x_{\bm m})^{\bm n} = \prod_{i=1}^d (x_i - ( x_{\bm m})_i)^{n_i}$. 

Then we can define a localized Taylor approximation $f_N$ of $f$:
\begin{align}\label{ch4:local_taylor_expansion}
\begin{aligned}
    f_N ( x )&= \sum_{\bm m\in \Lambda} h_{\bm m}( x) P_{\bm m}( x) \\
    &=\sum_{\bm m\in \Lambda}  \sum_{\|\bm n\|_1\leq r} \frac{\partial^{\bm n} f( x_{\bm m})}{\bm n!} h_{\bm m}( x)( x- x_{\bm m})^{\bm n}.
\end{aligned}
\end{align}
Applying the localization property of $h_{\bm m}$ and the Taylor expansion of $f$, we obtain the approximation error at $x \in A_s$
\begin{align*}
    |f( x)-f_N( x)| & = \left |\sum_{\bm m \in \Lambda} h_{\bm m}(\bm x)  \left(f( x)-P_{\bm m}( x) \right) \right | \\
    & \leq \sum_{\bm m \in \Lambda} h_{\bm m}(\bm x) \left | f( x)-P_{\bm m}( x)  \right | \\
    &\leq 2^d \frac{d^r}{r!} \left( \frac{d}{N} \right)^{r+\beta} ,
\end{align*}
where in the first step we use the definition of $\Lambda$ and the partition of unity and in the last step we use \eqref{eq:localapproximation} and again the fact that there are at most $2^d$ functions $h_{\bm m}(x) \neq 0$. 


Next, we want to construct a neural network $\phi$ that is built to approximate $f_N$ by approximating products.
Let
\begin{align*}
    \phi( x)  = \sum_{\bm m\in \Lambda} \sum_{\|\bm n\|_1 \leq r} a_{\bm m,\bm n}  \phi_{\bm m}( x ),
\end{align*}
where $a_{\bm m, \bm n} := \frac{\partial^{\bm n} f( x_{\bm m})}{\bm n!} $, $|a_{\bm m, \bm n}|  \leq 1$, and $\phi_{\bm m}( x )$ approximates the $d+r-1$ products within $h_{\bm m}( x)( x-  x_{\bm m})^{\bm n}$. In particular, we should have $\phi_{\bm m}( x )=0$ if $h_{\bm m}( x)( x-  x_{\bm m})^{\bm n} = 0$. For a given error tolerance $\delta$, in the proof of \cite[Theorem 1]{yarotsky2017error} it was shown that we can construct such a neural network $\phi_{\bm m}$ to approximate $P_{\bm m}$ with depth and total number of parameters no more than $c(r,d)|\ln\delta|$ for some constant $c(r,d)$.
Hence, we obtain the error bound of using $\phi$ to approximate $f_N$ on $A_s$
\begin{align*}
    \left| f_N( x) - \phi( x) \right| 
    &\leq \sum_{\bm m\in \Lambda} \sum_{\|\bm n\|_1 \leq r} |a_{\bm m,\bm n}| \left| h_{\bm m}( x)( x- x_{\bm m})^{\bm n} -  \phi_{\bm m}( x) \right| \\
    &\leq 2^d d^r \delta,
\end{align*}
where in the last step, we use the property that $x$ belongs to at most $2^d$ support of $h_{\bm m}(x)$.
Let us choose
\begin{align}\label{ch4:eq:number}
    N :=\left \lceil \left( \frac{r!\cdot \varepsilon}{2^{d+1} d^{2r+\beta}}\right)^{-1/(r+\beta)} \right \rceil, \quad \delta := \frac{\varepsilon}{2^{d+1}d^{r}}.
\end{align}
Then we can conclude that for any $x\in A_s$
\begin{align*}
    \left| f( x)-  \phi( x) \right| \leq \left| f( x)-  f_N( x) \right| + \left| f_N( x)-  \phi( x) \right| \leq \varepsilon,
\end{align*}
where $\phi$ has depth $O(|\ln \varepsilon|)$ and $O\left(2^d(N+1)^s \binom{d}{s}  d^r |\ln \varepsilon| \right) = O( \varepsilon^{-s/(r+\beta)}|\ln \varepsilon| )$ nonzero parameters. Setting $M := O( \varepsilon^{-s/(r+\beta)}|\ln \varepsilon| )$, we derive
\begin{align*}
    \begin{aligned}
        M^{-(r+\beta)/s} (\ln M)^{(r+\beta)/s} 
        &\asymp \left( \varepsilon^{-s/(r+\beta)}|\ln \varepsilon| \right)^{-(r+\beta)/s} \left( \ln \left( \varepsilon^{-s/(r+\beta)}|\ln \varepsilon| \right)  \right)^{(r+\beta)/s} \\
        &\asymp \varepsilon .
    \end{aligned}
\end{align*}
The proof is complete.
\end{proof}

Lemma~\ref{thm:yarotsky} demonstrates that when the inputs are sparse, the error bound depends on the sparsity rather than on the input dimension, indicating that sparsity can help mitigate the curse of dimensionality. 

The next lemma demonstrates that the above result can be straightforwardly generalized to arbitrary domains.
\begin{lemma}\label{lem:scale-func}
    Let $d \in \NN$, $B>0$, and $\beta \in (0,1]$. If $f \in C^{0,\beta}([-B,B]^d) $, then $f(2B \cdot) \in C^{0,\beta}([-\frac{1}{2},\frac{1}{2}]^d)$ and $\|f(2B \cdot)\|_{C^{0,\beta}([-\frac{1}{2}, \frac{1}{2}]^d)} \leq (1+ (2B)^\beta ) \|f( \cdot)\|_{C^{0,\beta}([-B, B]^d)} $. 
\end{lemma}
\begin{proof}
    Denote $\tilde{f}(\cdot) = f(2B \cdot)$. Obviously $\|f\|_{\C([-B,B]^d)} = \|\tilde f\|_{\C([-\frac{1}{2},\frac{1}{2}]^d)}$. 
    Notice that for any $x,y \in [-\frac{1}{2},\frac{1}{2}]^d$,
    \begin{align*}
        \sup_{x \ne y} \frac{|\tilde f( x) - \tilde f( y)|}{\|x - y\|_2^{\beta}} 
        = (2B)^\beta \sup_{x \ne y} \frac{| f(2B x) - f(2B y)|}{\|2B x - 2B y\|_2^{\beta}} = (2B)^\beta \sup_{u \ne v} \frac{| f( u) - f(v )|}{\| u -  v\|_2^{\beta}},
    \end{align*}
    where $u = 2Bx \in [-B,B]^d$ and $v = 2By\in [-B,B]^d$.
    Hence, by the definition of $\|\cdot\|_{C^{0,\beta}}$, we derive $\|f(2B \cdot)\|_{C^{0,\beta}([-\frac{1}{2}, \frac{1}{2}]^d)} \leq (1+ (2B)^\beta ) \|f( \cdot)\|_{C^{0,\beta}([-B, B]^d)} $.
\end{proof}

We are now prepared to prove Theorem~\ref{thm:holder}.

    
    

\begin{theorem}\label{thm:holder-full}
Let $m, s, K, M, N \in \mathbb{N}$, and let $B > 0$ and $\beta \in (0,1]$. 
Under Assumption~\ref{ass:sparsecoding}, for any nonlinear functional $\mathcal{P} \in C^{0,\beta}(\C(\Omega))_p$, 
there exists a functional neural network
\[
\Phi \in \nn\big(M \log(m+N),\, M(m+N)^2,\, \ln K,\, K\big)
\]
such that
\begin{align*}
    \sup_{f \in F}
    \left|
        \mathcal{P}(f) - \Phi(\widetilde{f}^m)
    \right|
    \le 
    \inf_{s <\bar{s}}
    \left\{
        C_{7}^{\beta} e^{-C_{5} \beta M}
        + C_{8}^{\beta} \big(\sigma_s(F,\mathcal{D}_N)_{L_\infty(\Omega,\nu)}\big)^{\beta}
        + C_9(K/\ln K)^{-\beta/s}
    \right\},
\end{align*}
where $\bar s$ is given in \eqref{notation:sbar} and the constants $C_5,C_7,C_8, C_9$ depend on $p$, $m$, $B$, $s$, $N$, and $\tilde{D}^m_N$ as given in \eqref{eq:C5}, \eqref{eq:C7}, \eqref{eq:C8}, and \eqref{eq:C9} 
in Appendix~\ref{subsec:cnn-rcvy}.

Moreover, the above error bound is valid as well for the $\|\cdot\|_{L_p(\Omega,\nu_\xi)}$ norm.
\end{theorem}

\begin{proof}[Proof of Theorem~\ref{thm:holder-full} and Theorem~\ref{thm:holder}]
    Theorem~\ref{thm:nn-estimator} (Theorem~\ref{thm:nn-estimator-full}) implies that there exists an estimator $f_s:= \sum_{i=1}^N \phi(\tilde{f}^m)_i u_i $ where $\phi(\tilde{f}^m)$ is the output of a CNN $\phi$ such that for any $f \in F$
    \begin{align}\label{eq:ftofs}
    \begin{aligned}
        |\P(f) - \P(f_s)| 
        &\leq \omega_{\P} \left( C_7  e^{- C_5 M} + C_8 \sigma_s(F, \D_N)_{L_\infty(\Omega,\nu)}  \right)_p  \\
        &\leq C_7^\beta  e^{- C_5 \beta M} + C_8^\beta \left( \sigma_s(F, \D_N)_{L_\infty(\Omega,\nu)} \right)^\beta,
    \end{aligned}
    \end{align}
    where we choose the kernel size of $\phi$ as $k=2$ and hence $\phi \in \nn_{\operatorname{CNN}}(M\log(m+N), M(m+N)^2)$. In addition, $\|\phi(\tilde{f}^m)\|_0\leq s$.

    
    Notice that $\sigma_s(f, \D_N)_{L_\infty(\Omega,\nu)} \leq \|f\|_{L_\infty(\Omega,\nu)} \leq B$ for any $f \in F$ and according to \eqref{eq:w_upper}, $\|L^{-1}w_{p,m}^s\|_\infty \leq 6Bm$ where $L $ is given in \eqref{eq:L}. Using Theorem~\ref{thm:nn-estimator}, we can derive an upper bound of $\phi(\tilde{f}^m)$
    \begin{align}
        \begin{aligned}
            \|\phi(\tilde{f}^m)\|_\infty 
            &\leq \|\phi(\tilde{f}^m) - w^s_{p,m}\|_\infty + \| w^s_{p,m} \|_\infty \\
            &\leq \|\phi(\tilde{f}^m) - w^s_{p,m}\|_p + \|L\|_\infty \| L^{-1} w^s_{p,m} \|_\infty \\
            &\leq C_4 + C_6 B + 6Bm \|L\|_\infty.
        \end{aligned}
    \end{align}
    Denote 
    \begin{align*}
        \bar{B} := C_4 + C_6 B + 6Bm \|L\|_\infty.
    \end{align*}
    Then $\phi(\tilde{f}^m) \in [-\bar B, \bar B]^N$.
    Since $f_s = \D_N \left(\phi(\tilde{f}^m) \right)$, $\P(f_s)$ is equivalent to
    \begin{align}\label{eq:p-phat}
        \hat{\P}\left(\phi(\tilde{f}^m) \right) :=  \P\circ \D_N\left(\phi(\tilde{f}^m) \right)
    \end{align}
    where $\phi(\tilde{f}^m) \in A_s$. Notice that Lemma~\ref{lem:modulus_equiv} implies
    \begin{align}\label{eq:fstonn}
    \begin{aligned}
        \omega_{\hat{\P}}(r;A_s ) \leq \omega_{\P}\left(\tilde c_1 r ; \Sigma_{s} (\D_N) \right)_p \leq \tilde c_1^\beta r^\beta.
    \end{aligned}
    \end{align}
    Let $C_{\P} := \sup_{w \in [-\bar B, \bar B]^N } |\hat\P(w)| $.
    Then $\| \hat{\P} \|_{C^{0,\beta}([-\bar B,\bar B]^N)} \leq \tilde c_1^\beta + C_{\P}$. Lemma~\ref{lem:scale-func} implies that
    $\| \hat{\P} (2\bar B \cdot ) \|_{C^{0,\beta}([-\frac{1}{2},\frac{1}{2}]^N)} \leq (\tilde c_1^\beta + C_{\P}) (1+2^\beta \bar B^\beta) $. Then we can utilize Lemma~\ref{thm:yarotsky} to construct a neural network $\psi$ with depth $O(\ln K)$ and at most $O(K)$ nonzero parameters for approximating $\hat{\P}(2 \bar B\cdot) $
    \begin{align}\label{eq:P-psi}
        \| \hat{\P} (2 \bar B \cdot ) - \psi(\cdot) \|_{C^{0,\beta}(A_s \cap [-\frac{1}{2},\frac{1}{2}]^N)} \leq (\tilde c_1^\beta + C_{\P}) (1+2^\beta \bar B^\beta)   (K/\ln K)^{-\beta/s}.
    \end{align}
    Combining \eqref{eq:p-phat} and \eqref{eq:P-psi}, we derive the error bound for approximating $\P(f_s)$
    \begin{align}\label{eq:fstonn}
        \begin{aligned}
            \left|\P(f_s) - \psi\left( \frac{1}{2\bar B} \phi(\tilde{f}^m)\right)\right| 
            &=  \left|\hat{\P}\left(2\bar B \left( \frac{1}{2\bar B}\phi(\tilde{f}^m) \right) \right) - \psi\left( \frac{1}{2\bar B} \phi(\tilde{f}^m)\right)\right| \\
            &\leq (\tilde c_1^\beta + C_{\P}) (1+2^\beta \bar B^\beta)   (K/\ln K)^{-\beta/s}.
        \end{aligned}
    \end{align}

    Let $\Phi:= \psi\left(\frac{1}{2\bar B}\phi(\cdot)\right)$. Combing \eqref{eq:ftofs} and \eqref{eq:fstonn}, we derive
    \begin{align*}
        |\P(f) - \Phi(\tilde{f}^m)| 
        &\leq |\P(f) - \P(f_s)| + |\P(f_s) - \Phi(\tilde f^m)|  \\
        &\leq 
        C_7^\beta  e^{- C_5 \beta M} + C_8^\beta \left( \sigma_s(F, \D_N)_{L_\infty(\Omega,\nu)} \right)^\beta + (\tilde c_1^\beta + C_{\P}) (1+2^\beta \bar B^\beta)   \left(\frac{K}{\ln K} \right)^{-\frac{\beta}{s}},
    \end{align*}
    and $\Phi \in \nn(M\log(m+N), M(m+N)^2,\ln K,K)$.
    We conclude the approximation error by introducing the notation
    \begin{align}\label{eq:C9}
        C_9 := (\tilde c_1^\beta + C_{\P}) (1+2^\beta \bar B^\beta),
    \end{align}
    where
    \begin{align*}
        \bar{B} &= C_4 + C_6 B + 6Bm \|L\|_\infty, \\
        \tilde c_1 &= C_1^{-1/p}  m^{-1/p} \|L^{-1}\|_2 \max\{1, m^{1/p-1/2}\} \sqrt{1+(2s-1) \mu( \tilde{D}_N^m)}, \\
        C_{\P} &= \sup_{w \in [-\bar B, \bar B]^N } |\hat\P(w)| .
    \end{align*}

    Applying Theorem~\ref{thm:sterm-ls} and proceeding analogously, the error bounds can be straightforwardly extended to the $\|\cdot\|_{L_p(\Omega,\nu_\xi)}$ setting.
\end{proof}

\section{Proofs of Section~\ref{subsec:sampling} and  Section~\ref{sec:examples}}\label{sec:pf_6}

\subsection{Proof of Lemma~\ref{lem:sparsity_estimation}}

To obtain an explicit error bound, one needs to estimate $\bar{s}$ and verify Assumption~\ref{ass:sparsecoding}. For a general dictionary, this is quite challenging. However, for orthonormal systems, the following lemma yields a detailed estimate. Beyond Lemma~\ref{lem:sparsity_estimation}, we also derive estimates for several additional constants that will be useful.

\begin{lemma}\label{lem:sparsity_estimation-full}
Assume that Assumptions \ref{ass:bndD} and \ref{ass:ortho} hold.
Let $\xi$ be a set of 
independent and identically distributed random points on $\Omega$ with distribution $\nu$.  
Then, with probability at least $1-\varepsilon$, the following statements hold whenever $m > \frac{64}{3\gamma} \log\left(2N^2/\varepsilon\right)$:
\begin{enumerate}
    \item Assumption \ref{ass:sparsecoding} holds with $C_1=\frac{1}{4}$ and $C_2=\frac{9}{4}$.
    \item The mutual coherence is bounded by \[\mu(\tilde{D}_N^{\,m}) 
\le
8\sqrt{\frac{ \log(2N^2/\varepsilon)}{3\gamma m}}.\]
\item The sparsity can be chosen as \[ s
=  \left\lfloor\frac{1}{2} \left( 1+ \frac{1}{16} \sqrt{\frac{3\gamma m}{\log(2N^2/\varepsilon)}} \right) \right\rfloor 
\le
\frac{1}{2}\left( 1 + \frac{1}{2 \mu(\tilde{D}^m_N)} \right),\]
where $\bar s$ is defined in \eqref{notation:sbar}.
\item The term $\sum_{t=1}^m \left| u_i(\xi_t)\right|^2$ is bounded by \[ \frac{1}{2} \gamma m \leq \sum_{t=1}^m \left| u_i(\xi_t)\right|^2 \leq \frac{3}{2}\gamma m .\]
\end{enumerate}

\end{lemma}

\begin{proof}[Proof of Lemma~\ref{lem:sparsity_estimation-full} and Lemma~\ref{lem:sparsity_estimation}]
    Let $\mathbb{P}(\xi \in A) := \nu(A)$ for any $A \subset \Omega$.
    We define $g_i(x) := u_i(x)/\sqrt{\gamma}$.
    Then it is easy to see that $\langle g_i, g_j \rangle_{L_2(\Omega)} = \delta_{ij}$ and $\|g\|_{C(\Omega)} \leq 1/\sqrt{\gamma}$. Define $\tilde{G}_\Lambda$ as the matrix that collects columns of $\tilde{G}^m_N := (\tilde{g}_1^m, \dots, \tilde{g}_N^m)$ with index set $\Lambda \subset [N]$ and $|\Lambda|=\lambda$. Then $\frac{1}{m} \tilde{G}_\Lambda^\top \tilde{G}_\Lambda$ is very close to an identity matrix \cite[Theorem 12.12]{Foucart2013AMI} with high probability 
    \begin{align*}
        \PP \left( \| \frac{1}{m} \tilde{G}_\Lambda^\top \tilde{G}_\Lambda - I \|_2 \geq \delta \right) \leq  2\lambda \cdot \exp\left\{-\frac{3 m \gamma \delta^2}{8\lambda}\right\}. 
    \end{align*}
    Denote $\|A\|_{\max} := \max_{ij}|A_{ij}|$ and let $\lambda = 2$. Obviously $\|A\|_{\max} \leq \|A\|_2$ . Hence,
    \begin{align*}
        \PP \left( \| \frac{1}{m} \tilde{G}_\Lambda^\top \tilde{G}_\Lambda - I \|_{\max} \geq \delta \right) \leq 4 \exp\left\{-\frac{3 m \gamma \delta^2}{16}\right\}. 
    \end{align*}
    Taking the union over all these $|\Lambda|=2$, we get
    \begin{align*}
        &\PP \left( \cup_{|\Lambda|=2,\Lambda\subset[N]}  \left\{ \| \frac{1}{m} \tilde{G}_\Lambda^\top \tilde{G}_\Lambda - I \|_{\max} \geq \delta \right\} \right) \\
        &\leq \sum_{|\Lambda|=2,\Lambda\subset[N]} 4\exp\left\{-\frac{3 m \gamma \delta^2}{16}\right\} \leq 2N^2 \exp\left\{-\frac{3 m \gamma \delta^2}{16}\right\}. 
    \end{align*}
    This implies
    \begin{align}\label{eq:largerpb}
        \PP \left( \cap_{|\Lambda|=2,\Lambda\subset[N]}  \left\{ \| \frac{1}{m} \tilde{G}_\Lambda^\top \tilde{G}_\Lambda - I \|_{\max} \leq \delta \right\} \right) \geq 1- 2N^2 \exp\left\{-\frac{3 m \gamma \delta^2}{16}\right\}. 
    \end{align}
    Notice that if for any $|\Lambda|=2$, the following inequality holds
    \begin{align*}
        \| \frac{1}{m} \tilde{G}_\Lambda^\top \tilde{G}_\Lambda - I \|_{\max} \leq \delta, 
    \end{align*}
    then we can obtain
    \begin{align}\label{eq:boundij}
    \begin{gathered}
        \left| \sum_{t=1}^m g_i(\xi_t) g_j(\xi_t) \right| \leq m \delta, \quad \forall i\neq j \\
         1-\delta \leq \frac{1}{m} \sum_{t=1}^m \left| g_i(\xi_t)\right|^2 \leq 1+\delta, \quad \forall i\in [N].
    \end{gathered}
    \end{align}
    Then following these bounds, we derive the upper bound
    \begin{align*}
        \begin{aligned}
            \frac{\left| \sum_{t=1}^m u_i(\xi_t) u_j(\xi_t) \right|}{   \sqrt{\sum_{t=1}^m \left| u_i(\xi_t)\right|^2 \sum_{t=1}^m \left| u_j(\xi_t)\right|^2}} = \frac{\left| \sum_{t=1}^m g_i(\xi_t) g_j(\xi_t) \right|}{   \sqrt{\sum_{t=1}^m \left| g_i(\xi_t)\right|^2 \sum_{t=1}^m \left| g_j(\xi_t)\right|^2}} \leq \frac{\delta}{1-\delta} \leq 2\delta,
        \end{aligned}
    \end{align*}
    where we take $\delta < \frac{1}{2}$ in order to obtain a nontrivial bound. This implies
    \begin{align}\label{eq:mcb}
        \mu(\tilde{D}^m_N) \leq 2 \delta.
    \end{align}
    We choose $c$ such that
    \begin{align*}
        s :=  \frac{1}{2}\left( 1 + \frac{1}{c \delta} \right) = \left\lfloor \frac{1}{2}\left( 1 + \frac{1}{4 \delta} \right) \right\rfloor .
    \end{align*}
    It is obvious that $c \geq 4$. 
    Hence, the following two inequalities hold
    \begin{align}\label{eq:smu}
    \begin{gathered}
        s \leq \frac{1}{2}\left( 1 + \frac{1}{ 2\mu(\tilde{D}^m_N)} \right)  < \frac{1}{2}\left( 1 + \frac{1}{ \mu(\tilde{D}^m_N)} \right),\\
        (2s-1)\mu(\tilde{D}^m_N) \leq (2s-1) 2 \delta =  \frac{2}{c} \leq \frac{1}{2} .
    \end{gathered}
    \end{align}
    Observe that with $L$ defined in \eqref{eq:L}, the matrix $\tilde{D}^m_N L$ has columns that are normalized. Applying Lemma~\ref{lem:mutual} with \eqref{eq:smu} to $\tilde{D}^m_N L$, we get
    \begin{align*}
        \frac{1}{2}\| w\|_2^2 \leq \| \tilde{D}^m_N L w \|_2^2 \leq  \frac{3}{2}  \| w\|_2^2 ,\quad \forall \|w\|_0 \leq 2s.
    \end{align*}
    This is equivalent to 
    \begin{align*}
        \frac{1}{2}\| L^{-1} w\|_2^2 \leq \| \tilde{D}^m_N  w \|_2^2 \leq  \frac{3}{2}  \| L^{-1} w\|_2^2 ,\quad \forall \| w\|_0 \leq 2s.
    \end{align*}
    Using \eqref{eq:boundij}, we can further remove $L^{-1}$
    \begin{align}\label{eq:sampling1}
        \frac{\gamma m}{4}\|  w\|_2^2 \leq \frac{1}{2}\| L^{-1} w\|_2^2 \leq \| \tilde{D}^m_N  w \|_2^2 \leq  \frac{3}{2}  \| L^{-1} w\|_2^2 \leq \frac{9\gamma m}{4}  \|  w\|_2^2 ,\quad \forall \| w\|_0 \leq 2s.
    \end{align}
    since $\delta < \frac{1}{2}$ and $g_i(x) := u_i(x)/\sqrt{\gamma}$.
    Define $f(x)=\sum_{i=1}^N w_i u_i(x)$. Notice that $\tilde{f}^m = \tilde{D}^m_N w$ and $ \gamma \|w\|_2^2 = \|f\|^2_{L_2(\Omega,\nu)}$. The property \eqref{eq:sampling1} is equivalent to
    \begin{align}\label{eq:sampling2}
			\frac{1}{4} \|f\|_{L_2(\Omega, \nu)}^2 \leq \frac{1}{m} \sum_{j=1}^{m} |f(\xi_j)|^2 \leq \frac{9}{4} \|f\|_{L_2(\Omega, \nu)}^2, \quad \forall f \in \Sigma_{2s}(\D_N).
    \end{align}
    Since 
    \begin{align*}
        \cap_{|\Lambda|=2,\Lambda\subset[N]}  \left\{ \| \frac{1}{m} \tilde{G}_\Lambda^\top \tilde{G}_\Lambda - I \|_{\max} \leq \delta \right\}
    \end{align*}
    implies the properties \eqref{eq:boundij}, \eqref{eq:mcb}, \eqref{eq:sampling2}, and \eqref{eq:smu}, by using \eqref{eq:largerpb}, we have that the following properties hold
    \begin{align*}
        \begin{gathered}
            \mu(\tilde{D}^m_N) \leq 2\delta, \\
            s =  \frac{1}{2}\left( 1 + \frac{1}{c \delta} \right) \leq \frac{1}{2}\left( 1 + \frac{1}{2 \mu(\tilde{D}^m_N)} \right),\\
            \frac{1}{4} \|f\|_{L_2(\Omega, \nu)}^2 \leq \frac{1}{m} \sum_{j=1}^{m} |f(\xi_j)|^2 \leq \frac{9}{4} \|f\|_{L_2(\Omega, \nu)}^2, \quad \forall f \in \Sigma_{2s}(\D_N),
        \end{gathered}
    \end{align*}
    with probability at least
    \begin{align*}
        1- 2N^2 \exp\left\{-\frac{3 m \gamma \delta^2}{16}\right\}.
    \end{align*}
    Or equivalently, by setting
    \begin{align*}
        \delta := \sqrt{\frac{16 \log(2N^2/\varepsilon)}{3\gamma m}} < \frac{1}{2},
    \end{align*}
    i.e., $m > \frac{64}{3 \gamma} \log \!\left( \frac{2N^2}{\varepsilon} \right)$, we find that the following properties hold with probability at least $1-\varepsilon$
    \begin{align}\label{eq:s_est}
        \begin{gathered}
            \mu(\tilde{D}^m_N) \leq 2 \sqrt{\frac{16 \log(2N^2/\varepsilon)}{3\gamma m}}, \\
            s = \left\lfloor\frac{1}{2} \left( 1+\frac{1}{16}\sqrt{\frac{3\gamma m}{\log(2N^2/\varepsilon)}} \right) \right\rfloor=  C \sqrt{\frac{\gamma m}{ \log(N^2/\varepsilon)}} \leq \frac{1}{2}\left( 1 + \frac{1}{2 \mu(\tilde{D}^m_N)} \right),\\
            \frac{1}{4} \|f\|_{L_2(\Omega, \nu)}^2 \leq \frac{1}{m} \sum_{j=1}^{m} |f(\xi_j)|^2 \leq \frac{9}{4} \|f\|_{L_2(\Omega, \nu)}^2, \quad \forall f \in \Sigma_{2s}(\D_N), \\
            \frac{1}{2} \gamma m \leq \sum_{t=1}^m \left| u_i(\xi_t)\right|^2 \leq \frac{3}{2}\gamma m ,
        \end{gathered}
    \end{align}
    for some absolute constant $C>0$.
\end{proof}

\subsection{Proof of Corollary~\ref{cor:main_decay}}
Define
$$ \operatorname{dist}(F, A_1^\alpha(\D_N))_{\infty} :=  \sup_{f \in F} \inf_{g\in A_1^\alpha(\D_N)} \|f-g\|_{\C(\Omega)} .$$ 
The next result describes the quality of sparse approximations of $F$ under the condition that $\operatorname{dist}(F, A_1^\alpha(\D_N))_{\infty} < \infty$.

\begin{lemma}[\cite{dai2023universal}]\label{lem:dai_proposition}
    Assume that $\D_N$ satisfies (2a) of Assumption~\ref{ass:DN} and $F\subset \C(\Omega)$ satisfies $\operatorname{dist}(F, A_1^\alpha(\D_N))_{\infty} < \infty$. Then for any $\xi \in \Omega^m$, there holds
    \begin{align*}
        \sigma_{2s}(F, \D_N)_{L_2(\Omega, \mu_\xi)} \leq s^{-\alpha-1/2} +  \operatorname{dist}(F, A_1^\alpha(\D_N))_{\infty}.
    \end{align*}
\end{lemma}

Lemma~\ref{lem:dai_proposition} describes the nonlinear approximation rate for functions whose coefficients decay rapidly, and this result holds for any choice of sample locations $\xi$. However, to obtain an explicit bound from Theorem~\ref{thm:holder}, we still need to bound the mutual coherence, which makes it challenging to relax the requirement of randomly sampled function values.


We are now ready to present the proof of Corollary~\ref{cor:main_decay}.

\begin{proof}[Proof of Corollary~\ref{cor:main_decay}]
    Let $p=2$.
    Lemma~\ref{lem:dai_proposition} together with Theorem~\ref{thm:holder-full} and Lemma~\ref{lem:sparsity_estimation-full} guarantee the existence of a functional neural network $\Phi \in \nn(M\log(m+N), M(m+N)^2, \log K, K)$ satisfying
    \begin{align*}
        \sup_{f\in F}|\P(f) - \Phi(\tilde{f}^m)| 
        \leq \inf_{s < \frac{1}{2}\left( 1+ \frac{1}{\mu(\tilde{D}^m_N)} \right)} \left\{ C_7^\beta e^{ -C_5 \beta M } + C_8^{\beta} 2^{(\alpha+1/2)\beta}  s^{-(\alpha+1/2)\beta} +C_9(K/\log K)^{-\frac{\beta}{s}} \right\}.
     \end{align*}
     The constants are given below
     \begin{align*}
        C_5 &= -\ln (2\mu(\tilde{D}^m_N) s-\mu(\tilde{D}^m_N)), \\
        C_7 &=  12 C_1^{-1/p} C_{m,p} m^{1-1/p} Bs,  \\
        C_8 &=  2^{1+1/p} C C_1^{-1/p} C_{m,p} m^{1-1/p} +C_3  , \\
        C_9 &= (\tilde c_1^\beta + C_{\P}) (1+2^\beta \bar B^\beta),
    \end{align*}
    where $C = 2 s \sum_{i=0}^{J} (2\mu(\tilde{D}^m_N) s-\mu(\tilde{D}^m_N))^i $, $C_{m,p}:= \max\{1, m^{1/p-1/2} \} \leq 1$, $C_3$ depends only on $p$, and
    \begin{align*}
        \bar{B} &= C_4 + C_6 B + 6Bm \|L\|_\infty, \\
        \tilde c_1 &= C_1^{-1/p}  m^{-1/p} \|L^{-1}\|_2 \max\{1, m^{1/p-1/2}\} \sqrt{1+(2s-1) \mu( \tilde{D}_N^m)}, \\
        C_{\P} &= \sup_{w \in [-\bar B, \bar B]^N } |\hat\P(w)| ,\\
        C_4 &= 6 c_0 B m s N^{1-1/p},  \\ 
        C_6 &= c_0 C 2^{1/p} m N^{1-1/p}, \\
        c_0 &= \max_{i\in[N]} \frac{1}{\sqrt{\sum_{t=1}^m \left| u_i(\xi_t)\right|^2}}.
    \end{align*}
    Next, we will determine these constants for use in subsequent calculations.
    For any $ f=\sum_{i}c_i u_i\in F $, the definition of $A_1^\alpha(\D_N)$ implies that 
    \begin{align*}
        \|f\|_{\C(\Omega)} \leq \sum_{i}|c_i|\|u\|_{\C(\Omega)} \leq \sum_{i}|c_i| \leq \sum_{i}|c_i|i^{\alpha} \leq 1.
    \end{align*}
    Hence the constant $B$ in Theorem~\ref{thm:holder-full} is given by
    \begin{align}\label{eq:cons-B}
        B := 1.
    \end{align}
    From Assumptions \ref{ass:bndD} and \ref{ass:ortho}, it follows directly that
    \begin{align}\label{eq:const-gamma}
        \gamma \leq 1.
    \end{align}
    From \eqref{eq:s_est} and \eqref{eq:const-gamma} we obtain
    \begin{align}\label{eq:cons-s}
        \begin{aligned}
            \sqrt{  \frac{\gamma m}{\log (N^2/\varepsilon)} } \asymp s &\leq \sqrt{  \frac{3m}{\log (N^2/\varepsilon)} },\\
            s &\leq \frac{1}{2}\left( 1 + \frac{1}{2 \mu(\tilde{D}^m_N)} \right) \\
            (2s-1)\mu(\tilde{D}^m_N) &\leq \frac{1}{2}.
        \end{aligned}
    \end{align}
    Since the inequality $\frac{1}{2} \gamma m \leq \sum_{t=1}^m \left| u_i(\xi_t)\right|^2 \leq \frac{3}{2}\gamma m$ holds due to Lemma~\ref{lem:sparsity_estimation-full}, we further have
    \begin{align}\label{eq:cons-c0}
        \begin{aligned}
            \|L\|_\infty &\lesssim  \frac{1}{\sqrt{\gamma m}} ,\\
            \|L^{-1}\|_2 &\lesssim  \sqrt{\gamma m},\\
            c_0 &\lesssim  \frac{1}{\sqrt{\gamma m}} .
        \end{aligned}
    \end{align}
    Notice that 
    \begin{align}\label{eq:cons-CP}
        \begin{aligned}
            C_{\P} 
            &\leq |\hat\P(0)| + \sup_{w \in [-\bar B, \bar B]^N }  |\hat\P(w) - \hat\P(0)|  \\
            &= |\P(0)| + \sup_{w \in [-\bar B, \bar B]^N }  |\P\circ \D_N (w) - \P(0)| \\
            &\leq |\P(0)| + \sup_{w \in [-\bar B, \bar B]^N }  \| \D_N (w) \|_{L_2(\Omega,\nu)}^\beta \\
            &\leq |\P(0)| + \sup_{w \in [-\bar B, \bar B]^N }  \| \sum_{i=1}^N w_i u_i \|_{L_\infty(\Omega,\nu)}^\beta \\
            &\lesssim  N^\beta \bar B^\beta.
        \end{aligned}
    \end{align}
    From \eqref{eq:cons-c0}, \eqref{eq:const-gamma}, \eqref{eq:cons-B}, \eqref{eq:cons-s}, and \eqref{eq:cons-CP}, we obtain the following estimates for the constants
    \begin{align*}
        \begin{aligned}
            c_0 &\lesssim (\gamma m)^{-1/2} \\
            C_4 &\lesssim \gamma^{-1/2} s \sqrt{m N},  \\ 
            C_6 &\lesssim \gamma^{-1/2} s \sqrt{m N}, \\
            \bar{B} &\lesssim \gamma^{-1/2} s \sqrt{m N}, \\
            \tilde c_1 &\lesssim   \gamma^{1/2} , \\
            C_{\P} &\lesssim \left( \gamma^{-1/2} s N^{\frac{3}{2}}\sqrt{m } \right)^\beta .
        \end{aligned}
    \end{align*}
    By inserting the above estimates into $C_5, C_7, C_8,$ and $C_9$, we arrive at the following result
    \begin{align*}
        \begin{aligned}
            C_5 &\geq \ln 2  ,\\
            C_7 &\asymp s\sqrt{m}   ,\\
            C_8 &\asymp  s\sqrt{m},\\
            C_9 &\lesssim \gamma^{-\beta} s^{2\beta} m^{\beta} N^{2\beta}  .
        \end{aligned}
    \end{align*}
    We can now rewrite the upper bound of $\sup_{f\in F}|\P(f) - \Phi(\tilde{f}^m)| $ in a simplified form as
    \begin{align*}
        \sup_{f\in F}|\P(f) - \Phi(\tilde{f}^m)| 
        \lesssim  (s\sqrt{m})^{\beta} e^{ -(\ln 2) \beta M } + (\sqrt{m})^{\beta} s^{-(\alpha-1/2)\beta}  \\
        + \gamma^{-\beta} s^{2\beta} m^{\beta} N^{2\beta} (K/\ln K)^{-\frac{\beta}{s}} .
     \end{align*}
     Choosing 
     \begin{align*}
         M&:=\left\lceil \frac{1}{(\ln 2) \beta} \ln \left( \left(\frac{K}{\ln K}\right)^{\beta/s}  \right) \right\rceil,
     \end{align*}
     then we can combine the first and last terms to obtain a simpler upper bound
     \begin{align}\label{eq:ineupper}
     \begin{aligned}
         \sup_{f\in F}|\P(f) - \Phi(\tilde{f}^m)| 
        \lesssim   (\sqrt{m})^{\beta} s^{-(\alpha-1/2)\beta}  
        + \gamma^{-\beta} s^{2\beta} m^{\beta} N^{2\beta} (K/\ln K)^{-\frac{\beta}{s}} .
     \end{aligned}
     \end{align}
     Setting $N := 2m$, $\varepsilon := 1/m$, and observing from \eqref{eq:cons-s} that $s< \sqrt{m/\log m}$, we further derive
     \begin{align}\label{eq:}
     \begin{aligned}
        &\sup_{f\in F}|\P(f) - \Phi(\tilde{f}^m)| 
         \\
        &\lesssim   \underbrace{m^{-\frac{\beta}{4}(2\alpha-3)} (\log m)^{\frac{\beta}{4}(2\alpha-1)}  }_{(i)}
        + \underbrace{ m^{4\beta} (\log m)^{-\beta}  \left(\frac{K}{\ln K}\right)^{-\frac{\beta}{s}} }_{(ii)} .
     \end{aligned}
     \end{align}
     We need to assume $\alpha > \frac{3}{2}$ so that term $(i)$ tends to zero as $m$ goes to infinity.

     To derive the relationship between $m$ and $K$, we take logarithms on both sides of the error bound \eqref{eq:ineupper} and then establish corresponding upper bounds to control them individually
     \begin{align*}
         (i)  & \Rightarrow -\frac{\beta}{4}(2\alpha-3)\log m + \frac{\beta}{4}(2\alpha-1)\log\log m ,\\
         (ii) & \Rightarrow 4\beta \log m - \beta\log \log m - \beta \sqrt{\frac{\log m}{m}}  (\log K-\log\log K ).
     \end{align*}
    Choose $m, K$ so that 
    \begin{align*}
        \log K = \left(\frac{1}{4}(2\alpha-3)+ 4 \right) \sqrt{m \log m} .
    \end{align*}
    Consequently, 
    \begin{align}
        \begin{aligned}
            m &\asymp \frac{(\log K)^2}{ \log \log K},\\
            \log m &\asymp \log \log K,
        \end{aligned}
    \end{align}
    and the above two terms $(i)$ and $(ii)$ after taking logatithms can be controlled by
    \begin{align*}
          -\frac{\beta}{4}(2\alpha-3)\log m + \frac{\beta}{4}(2\alpha-1)\log\log m 
          &\asymp -\frac{\beta}{4}(2\alpha-3) \log m ,\\
          4\beta \log m - \beta\log \log m - \beta \sqrt{\frac{\log m}{m}}  (\log K-\log\log K ) 
          &\asymp  -\frac{\beta}{4}(2\alpha-3) \log m .
    \end{align*}
    This implies that
    \begin{align*}
        \begin{aligned}
            \sup_{f\in F}|\P(f) - \Phi(\tilde{f}^m)| 
            &\lesssim   m^{-\frac{\beta}{4}(2\alpha-3)} (\log m)^{\frac{\beta}{4}(2\alpha-1)} \\
            &\lesssim  (\log K)^{-\beta(\alpha - \frac{3}{2})} (\log\log K)^{\beta(\alpha-1)}.
        \end{aligned}
    \end{align*}
    In addition, we have
    \begin{align*}
    \begin{aligned}
        \varepsilon &\asymp \frac{ \log \log K}{(\log K)^2}, \\
        M &\asymp \frac{1}{s}\log K \lesssim \log\log K ,\\
        M\log(m+N) &\lesssim  \left(\log\log K \right)^2 , \\
        M(m+N)^2 &\lesssim (\log K)^2,
    \end{aligned}
    \end{align*}
    where we use \eqref{eq:cons-s} to estimate $M$.
     The proof is complete.
\end{proof}

\subsection{Proof of Corollary~\ref{cor:main_mixed_sobolev}}
This proof follows a similar approach to the computation presented in Corollary~\ref{cor:main_decay}. To ensure completeness, we include the detailed calculation below.

\begin{proof}[Proof of Corollary~\ref{cor:main_mixed_sobolev}]
    Let $p=2$.
    Theorem~\ref{thm:holder-full} and Lemma~\ref{lem:sparsity_estimation-full} together with the approximation results for $W^{a,b}_A$ in \cite[Lemma 2.1]{temlyakov2015constructive} guarantee the existence of a functional neural network $\Phi \in \nn(M\log(m+N), M(m+N)^2, \log K, K)$ satisfying
    \begin{align*}
        &\sup_{f\in F}|\P(f) - \Phi(\tilde{f}^m)| 
        \\
        &\leq \inf_{s < \frac{1}{2}\left( 1+ \frac{1}{\mu(\tilde{D}^m_N)} \right)} \left\{ C_7^\beta e^{ -C_5 \beta M } + C_8^{\beta} s^{- (a +1/2)\beta} (\log s)^{\beta(d-1)(a+b)+\beta/2} +C_9(K/\log K)^{-\frac{\beta}{s}} \right\}.
     \end{align*}
        Although the original result \cite[Lemma 2.1]{temlyakov2015constructive} is built with exponential form of trigonometric functions, which is complex, it can be equivalently reformulated as a series of real $\sin$ and $\cos$ functions, our results are still valid.

     To estimate the constants $C_5, C_7, C_8, C_9$, we proceed in the same way as in the proof of Corollary~\ref{cor:main_decay}. Since the argument relies only on Lemma~\ref{lem:sparsity_estimation-full}, the constants do not depend on the input functions. Therefore, we can directly adopt the estimates from the proof of Corollary~\ref{cor:main_decay} and obtain
    \begin{align*}
        \begin{aligned}
            C_5 &\geq \ln 2  ,\\
            C_7 &\asymp s\sqrt{m}   ,\\
            C_8 &\asymp  s\sqrt{m},\\
            C_9 &\lesssim \gamma^{-\beta} s^{2\beta} m^{\beta} N^{2\beta}   .
        \end{aligned}
    \end{align*}
    We can now rewrite the upper bound of $\sup_{f\in F}|\P(f) - \Phi(\tilde{f}^m)| $ in a simplified form as
    \begin{align*}
        \sup_{f\in F}|\P(f) - \Phi(\tilde{f}^m)| 
        \lesssim  (s\sqrt{m})^{\beta} e^{ -(\ln 2) \beta M } + (\sqrt{m})^{\beta} s^{- (a -1/2)\beta} (\log s)^{\beta(d-1)(a+b)+\beta/2}  \\
        + \gamma^{-\beta} s^{2\beta} m^{\beta} N^{2\beta}  (K/\ln K)^{-\frac{\beta}{s}} .
     \end{align*}
     Choosing 
     \begin{align*}
         M&:=\left\lceil \frac{1}{(\ln 2) \beta} \ln \left( \left(\frac{K}{\ln K}\right)^{\beta/s}  \right) \right\rceil,
     \end{align*}
     then we can combine the first and last terms to obtain a simpler upper bound
     \begin{align}\label{eq:ineupper-mix}
     \begin{aligned}
         &\sup_{f\in F}|\P(f) - \Phi(\tilde{f}^m)| 
         \\
        &\lesssim   (\sqrt{m})^{\beta} s^{- (a -1/2)\beta} (\log s)^{\beta(d-1)(a+b)+\beta/2}   
        + \gamma^{-\beta} s^{2\beta} m^{\beta} N^{2\beta}  (K/\ln K)^{-\frac{\beta}{s}} .
     \end{aligned}
     \end{align}
     Setting $N := 2m$, $\varepsilon := 1/m$, and observing from \eqref{eq:cons-s} that $s< \sqrt{m/\log m}$, we further derive
     \begin{align}\label{eq:}
     \begin{aligned}
        &\sup_{f\in F}|\P(f) - \Phi(\tilde{f}^m)| 
         \\
        &\lesssim   \underbrace{m^{-\frac{\beta}{4}(2a-3)} (\log m)^{\frac{\beta}{4}(2a-1)+ \beta(d-1)(a+b)+\frac{\beta}{2}}  }_{(i)}
        + \underbrace{ m^{4\beta} (\log m)^{-\beta}  \left(\frac{K}{\ln K}\right)^{-\frac{\beta}{s}} }_{(ii)} .
     \end{aligned}
     \end{align}
     We need to assume $a > \frac{3}{2}$ so that term $(i)$ tends to zero as $m$ goes to infinity.

     To derive the relationship between $m$ and $K$, we take logarithms on both sides of the error bound \eqref{eq:ineupper-mix} and then establish corresponding upper bounds to control them individually
     \begin{align*}
         (i)  & \Rightarrow -\frac{\beta}{4}(2a-3)\log m + \left(\frac{\beta}{4}(2a-1) + \beta(d-1)(a+b)+\frac{\beta}{2} \right)  \log\log m ,\\
         (ii) & \Rightarrow 4\beta \log m - \beta\log \log m - \beta \sqrt{\frac{\log m}{m}}  (\log K-\log\log K ).
     \end{align*}
    Choose $m, K$ so that 
    \begin{align*}
        \log K = \left(\frac{1}{4}(2a-3)+ 4 \right) \sqrt{m \log m} .
    \end{align*}
    Consequently, 
    \begin{align}
        \begin{aligned}
            m &\asymp \frac{(\log K)^2}{ \log \log K},\\
            \log m &\asymp \log \log K,
        \end{aligned}
    \end{align}
    and the above two terms $(i)$ and $(ii)$ after taking logatithms can be controlled by $-\frac{\beta}{4}(2a-3) \log m$.
    This implies that
    \begin{align*}
        \begin{aligned}
            \sup_{f\in F}|\P(f) - \Phi(\tilde{f}^m)| 
            &\lesssim   m^{-\frac{\beta}{4}(2a-3)} (\log m)^{\frac{\beta}{4}(2a-1)+ \beta(d-1)(a+b)+\frac{\beta}{2}} \\
            &\lesssim  (\log K)^{-\beta(a - \frac{3}{2})} (\log\log K)^{\beta(a+(d-1)(a+b)-\frac{1}{2})}.
        \end{aligned}
    \end{align*}
    In addition, we have
    \begin{align*}
    \begin{aligned}
        \varepsilon &\asymp \frac{ \log \log K}{(\log K)^2}, \\
        M &\asymp \frac{1}{s}\log K \lesssim \log\log K ,\\
        M\log(m+N) &\lesssim  \left(\log\log K \right)^2 , \\
        M(m+N)^2 &\lesssim (\log K)^2,
    \end{aligned}
    \end{align*}
    where we use \eqref{eq:cons-s} for estimating $M$.
     The proof is complete.
\end{proof}

\end{document}